\documentclass{article}
\usepackage{tikz}
\usepackage{amsmath,amssymb,amsthm,enumerate}
\usepackage{enumitem,hyperref}
\usepackage[utf8]{inputenc}
\usepackage{geometry}
\usepackage{aliascnt,bbm}

\def\rset{\mathbb R}

\def\tk{\theta_k}
\def\tn{\theta_n}
\def\tnp{\theta_{n+1}}
\def\ttn{\widetilde{\theta}_n}
\def\ttnp{\widetilde{\theta}_{n+1}}
\def\te{\widetilde{\varepsilon}}
\def\wk{w_k}
\def\wn{w_n}
\def\wnp{w_{n+1}}
\def\twn{\widetilde{w}_n}
\def\twnp{\widetilde{w}_{n+1}}
\def\Fn{\mathcal{F}_n}
\def\dn{e_n}
\def\k{\Gamma}

\def\min{\mathrm{min}}
\def\max{\mathrm{max}}

\def\eqsp{\;} 

\def\1{\mathbbm{1}} 

\def\PE{\mathbb{E}} 
\def\F{\mathcal{F}} 

\newcommand{\eqdef}{\ensuremath{\stackrel{\mathrm{def}}{=}}} 

\newcommand\sequence[3]{\ifthenelse{\equal{#3}{}}{\ensuremath{\{
#1_{#2}\}}}{\ensuremath{\{ #1_{#2}, \eqsp #2 \in #3 \}}}} 
\newcommand\sequenceup[3] {\ifthenelse{\equal{#3}{}}{\ensuremath{\{
      #1_{#2}\}}}{\ensuremath{\{ #1^{#2}, \eqsp #2 \in #3 \}}}} 
\newcommand{\pos}[1]{\lfloor#1\rfloor_{+}} 

\newcommand{\oo}[1]{#1^{\odot 2}} 

\newcommand{\cu}{c_{1,\varepsilon}}
\newcommand{\cd}{c_{2,\varepsilon}}
\newcommand{\ct}{c_{3,\varepsilon}}

\newcommand{\Hp}{\mathbf{H}_{\sigma}^p}
\newcommand{\Hf}{\mathbf{H}_f}
\newcommand{\Hg}{\mathbf{H}_{\text{Steps}}}
\newcommand{\led}{\lesssim_{d}}
\newcommand{\gtd}{\gtrsim_{d}}

\theoremstyle{plain}
\newtheorem{theorem}{Theorem}

\newaliascnt{proposition}{theorem}
\newtheorem{proposition}[proposition]{Proposition}
\aliascntresetthe{proposition}
\newaliascnt{lemma}{theorem}
\newtheorem{lemma}[lemma]{Lemma}
\aliascntresetthe{lemma}
\newaliascnt{corollary}{theorem}

\aliascntresetthe{corollary}

\newaliascnt{definition}{theorem}
\newtheorem{definition}[definition]{Definition}
\aliascntresetthe{definition}
\newaliascnt{remark}{theorem}
\newtheorem{remark}[remark]{Remark}
\aliascntresetthe{remark}

\usepackage{xcolor}

\author{S\'{e}bastien Gadat$^{1,2}$, Ioana Gavra$^{3}$   
\\  $^{1}$Toulouse School of Economics (CNRS UMR 5314), Universit\'{e} Toulouse I Capitole \\ $^{2}$ Institut Universitaire de France\\$^{3}$ IRMAR (UMR CNRS 6625), Universit\'e de Rennes }

\date{\today}

\begin{document}

\title{Asymptotic study of stochastic adaptive algorithm in non-convex landscape}
\maketitle

\begin{abstract}
This paper studies some asymptotic properties of adaptive algorithms widely used in optimization and machine learning, and among them Adagrad and Rmsprop, which are involved in most of the blackbox deep learning algorithms. Our setup is the non-convex landscape optimization point of view, we consider a one time scale parametrization and we consider the situation where these algorithms may be used or not with mini-batches. We adopt the point of view of stochastic algorithms and establish the almost sure convergence of these methods when using a decreasing step-size towards the set of critical points of the target function. With a mild extra assumption on the noise, we also obtain the convergence towards the set of minimizers of the function. Along our study, we also obtain a ``convergence rate'' of the methods, in the vein of the works of \cite{GhadimiLan}.

\end{abstract}




\section{Introduction}
\subsection{Stochastic optimization}
Minimizing a differentiable non-convex function $f: \mathbb{R}^d \longrightarrow \mathbb{R}$ when $f$ is defined through an expected loss  in a  statistical model
is a common way of estimation from an empirical set of observations in nowadays machine learning problems. In particular, some difficult optimization is generally involved in neural networks learning, see \textit{e.g.} \cite{bottou2018optimization} where the major challenge of a such problem is the large scale statistical settings (large number of observations $n$ involved in the definition of $f$ and large dimension of the ambient space $d$) and the non-convex landscape property when using a cascade of logistic regressions. We consider in this work the generic formulation: 
$$
\forall \theta \in \mathbb{R}^d: \qquad f(\theta) = \mathbb{E}_{X \sim \mathbb{P}}[\tilde{f}(\theta,X)],
$$
where $X$ is a random variable sampled according to an \textit{unknown} distribution $\mathbb{P}$. To perform the optimization of $f$ under the uncertainty on $\mathbb{P}$, we assume that we can compute all along the process of our algorithm some noisy but unbiased approximations of the gradient of $f$ computed at the current point of the algorithm. One typical example of a such algorithm is the so-called Stochastic Gradient Descent (SGD) introduced in the famous work of \cite{RobbinsMonro}, which is governed by the stochastic evolution:
$$
\theta_{k+1} = \theta_k - \gamma_{k+1} \nabla_{\theta} \tilde{f}(\theta_k,X_{k+1}),
$$
whose early success in the sixties has been at least rejuvenated if not  resurrected with the development of massive learning problems, in the last fifteen years. We refer among other to \cite{bottou-bousquet-2008,moulines2011non} or to \cite{JMLR:v15:bach14a} and the references therein for various applications in machine learning. Although being one of the state-of-the-art method to handle massive datasets, SGD suffers from several issues: difficulty to tune the step-size sequence or dependence on the gradient flow that may be lazy in flat areas, which is especially the case when looking at non-convex neural network problems. 

Some popular improvements are commonly patched to the vanilla SGD, and among others we refer to the popular acceleration obtained with the Polyak-Ruppert averaging \cite{polyakjuditsky,Ruppert,moulines2011non,cardot2017online,Gadat-Panloup}, variance reduction with mini-batch strategies (see \textit{e.g.} \cite{le2012stochastic,johnson2013accelerating}). 

While these two last improvements do not modify the underlying gradient flow, other strategies rely on a modification of the dynamical system exploiting acceleration brought by momentum with second order terms. The first historical example is  the Heavy Ball with Friction optimization based on the seminal contribution \cite{Polyak1964} and then translated into a stochastic framework (see \textit{e.g.} \cite{GPS2018,sebbouh2020convergence,loizou2020momentum}). Another second example is the Nesterov Accelerated Gradient Descent (see \textit{e.g.} \cite{nesterov1983method}) translated and studied in the noisy situation recently in a large number of works (see among other  \cite{ghadimi2016accelerated,jin2018accelerated}).

A last stimulating subject of investigation we refer to for improving the behaviour of stochastic algorithms rely on adaptive methods: they consist in tuning the step-size sequence either with a per-coordinate strategy or with a matricial inversion in front of the gradient $\nabla_\theta \tilde{f}(\theta_k,X_{k+1})$. Among other, Adagrad  introduced in \cite{Adagrad} (with a long range memory of past gradients) and Rmsprop (with an exponential moving average) taught in \cite{hinton2012neural} are typical examples of step-size adaptation with second-order moments learned on-line and these two algorithms are at the core of our work. Another state-of-the-art algorithm is the ADAM method introduced in \cite{Adam} and used in GAN optimization \cite{goodfellow2014generative}. These algorithms are referred to as \textit{adaptive methods} and have encountered a striking raise of attention these recent years in machine learning (see \textit{e.g.} \cite{Adagrad_Ward} \cite{Adam_Zou}). In the statistical community, stochastic Newton and stochastic Gauss-Newton methods may also be seen as adaptive algorithms with a direct matricial inversion and multiplication: these methods have shown both good theoretical and numerical abilities for regressions \cite{cnac2020efficient}, logistic regression \cite{Bercu_Godichon_Portier2020}, average consensus research \cite{loizou2020momentum} or optimal transport problems
\cite{Siviero}.\\

To the best of our knowledge, there is little convergence mathematical results on adaptive algorithms: 
\cite{SilvaGazeau} studies the deterministic dynamical system behind adaptive algorithms and obtained long-time behaviour of the trajectories or the value function following ideas of \cite{cabot2009long,SuBoydCandes}. More recently, 
 \cite{BarakatBianchi} (that is more closely related to us) obtains the almost sure convergence of their algorithms towards critical points with a parametrization that is different from our but the authors leave as an open problem the important question of the convergence towards a \textit{minimizer} of $f$.\footnote{The same week we sent our paper on Arxiv, \cite{BarakatBianchi2} also published some results on the trap avoidance of adaptive algorithms but not consider the mini-batch effect that is known to be a crucial ingredient for the efficiency of adaptive methods. We refer to  Theorem \ref{thm: cv quanti} and \ref{thm:conv min} below for the conditions we obtained on the mini-batch sequence.}
Finally, some recent contributions in machine learning \cite{Adagrad_Ward,Adam_Zou,Adam_Bach} address some ``convergence'' questions for adaptive algorithms with constant step-size. They provide a non-asymptotic study with a  step-size that is tuned according to the finite horizon of simulation. Even though these results are of major interest from a numerical point of view, they do not really answer the question of convergence from a trajectorial point of view (see Section \ref{sec:param} below). 
 The objective of this work may be seen as modest at the moment: we aim to study the asymptotic behaviour of Adagrad and Rmsprop \textit{i.e.} we aim to show the almost sure convergence towards a local minimizer of the objective function $f$. However limited at first sight, we will see that the convergence of the trajectories outside local traps is already challenging, especially when a mini-batch strategy is used.
\section{Adaptive algorithms and main results}
\subsection{Definition of the methods}
The algorithm we consider in this paper use the vectorial division/multiplication notations introduced in Adagrad (see \cite{Adagrad}) and now widely used in machine learning.
The vectorial division  $\frac{u}{v}$ and multiplication $ u\cdot v$ are the coordinate per coordinate operations introduced by:
$$
\left(\frac{u}{v}\right)_i = \frac{u_i}{v_i},\quad \mbox{and} \quad  \left(u \cdot v\right)_i = u_iv_i\quad \forall i \in \{1,\ldots, d\}
$$
In the meantime,  the notation $\oo{u}$ corresponds to the coordinate per coordinate square:
$$
\forall i \in \{1,\ldots, d\} \qquad
\{\oo{u}\}_{i} = u_i^2,
$$
whereas $\sqrt{u}$  denotes the coordinate per coordinate square root:
$$ \left(\sqrt{u}\right)_i= \sqrt{u}_i, \quad \forall i \in \{1,\ldots, d\}$$
Finally, the sum of a vector  $u\in \mathbb{R}^d$ and a scalar $\varepsilon\in \mathbb{R}$ is given by :
$$ \left(u+\varepsilon\right)_i=u_i+\varepsilon \quad \forall i \in \{1,\ldots, d\}.$$
Following the recent work of \cite{SilvaGazeau}, we consider the joint evolution of $(\theta_n,w_n)_{n \ge 1}$ in $\mathbb{R}^d \times \mathbb{R}^d$ of a stochastic algorithm defined by:
\begin{equation}\label{def:adam}
\begin{cases}
\theta_{n+1}=\theta_n-\gamma_{n+1}\dfrac{g_{n+1}}{\sqrt{w_n+\varepsilon}}\\
w_{n+1}=w_n+ \gamma_{n+1}(p_n \oo{g_{n+1}} - q_n w_n )
\end{cases},
\end{equation}
where $(g_n)_{n \ge 1}$ corresponds to a noisy stochastic evaluation of the gradient of the function $f$, corrupted by an additive noise sequence $(\xi_{n+1})_{n \ge 1}$:	
$$
g_{n+1}= \nabla f (\theta_n)+\xi_{n+1}.
$$ 
Following the initial vectorial notations, we emphasize that \eqref{def:adam} means that for any coordinate $i \in \{1,\ldots,d\}$, the position $(\theta_n)_{n \ge 1}$ and scaling factor $(w_n)_{n \ge 1}$ are updated according to:
$$
\begin{cases}
\theta_{n+1}^i = \theta_n^i - \gamma_{n+1} \frac{g_{n+1}^i}{\sqrt{w_n^i+\varepsilon}}\\
w_{n+1}^i = w_n^i + \gamma_{n+1} (p_n (g_{n+1}^i)^2-q_n w_n^i)
\end{cases}.
$$

\subsection{Link with other parametrizations\label{sec:param}}
\subsubsection{Historical parametrization}

We discuss here on our choice of the Adagrad/Rmsprop parametrization \eqref{def:adam} using the one of \cite{SilvaGazeau}, and its link with the standard parametrization introduced in \cite{Adagrad} or \cite{hinton2012neural} and used in later works \cite{Adam_Bach,Adam_Zou} for Adam and in \cite{Adagrad_Ward} for Adagrad.
We have chosen to use this formulation, which is inspired from the limiting O.D.E. of the continuous time adaptive gradient system following previous works on  accelerated or second order dynamics and among other we refer to Memory gradient diffusion \cite{gadat2014long}, Ruppert-Polyak averaging \cite{Gadat-Panloup}, Heavy Ball systems \cite{Attouch,cabot2009long,cabot2009second,GPS2018} or more generally Nesterov acceleration \cite{Nesterov04,SuBoydCandes,Attouch_nesterov} and dissipative systems \cite{haraux1991systemes,alvarez}.

The pioneering works \cite{Adagrad} and \cite{hinton2012neural} use the following parametrization:
\begin{equation}\label{def:adam_bach}
\begin{cases}
\ttnp=\ttn-\alpha_{n+1}\dfrac{g_{n+1}}{\sqrt{v_n+\te}}\\
v_{n+1}=\beta_2(n) v_n+ \oo{g_{n+1}} 
\end{cases}
\end{equation} 
for $\beta_2 \in (0,1)$ and 
when no heavy ball momentum (see \textit{e.g.} \cite{Polyak1964}) is used in the algorithm (which is also the case we are considering in this work).
Notice that $\beta_2$ may depend on the current iteration in a second stage and for the sake of completeness, we consider a general sequence $(\beta_2(n))_{n \ge 1}$.

We introduce the natural normalizing sequence $(S_n)_{n \ge 1}$, defined by
$S_0=1$ and the following recursion:
\begin{equation}\label{def:Sn}
S_{n+1} = \beta_2(n) S_n+1
\end{equation}
We are led to introduce $\twn=v_n/S_n$ and $\te_{n}=\te/S_n$ and we observe that:
$$
\ttnp=\ttn-\frac{\alpha_{n+1}}{\sqrt{S_n}} \frac{g_{n+1}}{\sqrt{\twn+\te_n}},
$$
whereas the second coordinate evolves according to:
$$
\twnp=  \frac{\beta_2(n) v_n + \oo{g_{n+1}}}{S_{n+1}} = \frac{\oo{g_{n+1}}}{S_{n+1}} + \twn \frac{\beta_2(n) S_n}{S_{n+1}} = \twn + \frac{1}{S_{n+1}}\left[  \oo{g_{n+1}} - \twn \right].
$$
Following the recommendation of \cite{Adagrad_Ward,Adam_Bach} (see in particular Equation (2.4) of \cite{Adam_Bach}), we can introduce a new step-size sequence $(\tilde{\alpha}_{n})_{n \ge 1}$ such that  $\alpha_{n+1}=\widetilde{\alpha}_{n+1} \sqrt{S_n}$ and  we recover in this case a joint evolution:
\begin{equation}\label{eq:reparam}
\begin{cases}
\ttnp = \ttn - \widetilde{\alpha}_{n+1} \frac{g_{n+1}}{\sqrt{\twn+\te_n}}\\
\twnp=\twn+\frac{1}{S_{n+1}}[\oo{g_{n+1}} -\twn]
\end{cases}.
\end{equation}
\subsubsection{Two-time scale parametrization\label{sec:two_time_scale}}

We deduce  that the popular parametrization introduced in the seminal contributions of Adagrad, Rmsprop or ADAM and the one we used in our paper are equivalent and
fall into our framework described in Equation \eqref{def:adam} with possibly two time-scales on the system $(\ttn,\twn)_{n \ge 1}$:
\begin{equation}\label{eq:two_time_scales}
\begin{cases}
\ttnp=\ttn-\gamma_{n+1}\dfrac{g_{n+1}}{\sqrt{w_n+\varepsilon}}\\
\twnp=\twn+ \widetilde{\gamma_{n+1}}(p_n \oo{g_{n+1}} - q_n \twn ),
\end{cases}
\end{equation}
where 
\begin{equation}
\label{eq:step_size_link}
\gamma_{n+1}= \frac{\alpha_{n+1}}{\sqrt{S_n}} \quad \text{and} \quad \widetilde{\gamma}_{n+1} = \frac{1}{S_{n+1}} \quad \text{and} \quad  p_n=q_n=1.
\end{equation}
The choice of the sequence $(\beta_2(n))_{n \ge 1}$ is key to understand the stochastic algorithm we obtain in \eqref{eq:reparam}.

$\bullet$ Case constant $\beta_2 =1$ (Adagrad of \cite{Adagrad}) This case is certainly the easiest to understand since the natural rescaling $S_n$ of the sequence $(v_n)_{n \ge 1}$ is $S_n =  n$. In this case, 
we recover a joint evolution:
\begin{equation}\label{eq:reparam_1}
\begin{cases}
\ttnp = \ttn - \frac{\alpha_{n+1}}{\sqrt{n}} \frac{g_{n+1}}{\sqrt{\twn+\te_n}}\\
\twnp=\twn+\frac{1}{n+1}[\oo{g_{n+1}} - \twn]
\end{cases},
\end{equation}
which entails $\gamma_{n+1}=\frac{\alpha_{n+1}}{\sqrt{n}}$ and $\widetilde{\gamma}_{n+1}= \frac{1}{n+1}$.
With the choice of \cite{Adam_Bach}, we then obtain a constant step-size stochastic algorithm for the coordinate $(\tn)_{n \ge 1}$ associated with a uniform Cesaro averaging on the sequence of past squared gradients $(\oo{g_{n+1}})_{n \ge 1}$.

$\bullet$ Case constant $\beta_2 \in (0,1)$ (Adam of \cite{Adam} with no momentum).
Since $S_n$ converges exponentially fast towards $(1-\beta_2)^{-1}$, the system is close to:
\begin{equation}\label{eq:reparam_2}
\begin{cases}
\ttnp = \ttn - (1-\beta_2) \alpha_{n+1}\frac{g_{n+1}}{\sqrt{\twn+\te_n}}   \\
\twnp= \twn+(1-\beta_2)[ \oo{g_{n+1}} - \twn] 
\end{cases},
\end{equation}
which entails $\gamma_{n+1}=(1-\beta_2) \alpha_{n+1}$ and $\widetilde{\gamma}_{n+1}=(1-\beta_2)$.

$\bullet$ Case $\beta_2(n) = 1-b n^{-\beta}$ with $b \in (0,1)$. This last case where the sequence goes to $1$ with $n$ corresponds to an intermediary situation between $\beta_2=1$ and $\beta_2<1$, this transition being parametrized by $\beta \in [0,+\infty]$. We shall introduce the sequence of products:
$$
\pi_k = \beta_2(1) \ldots \beta_2(k) \qquad \text{with} \qquad \pi_0=1,
$$
and a straightforward computation yields
$$
S_{n+1} = \sum_{k=0}^n \pi_n \pi_k^{-1}.
$$
It is well known (see \textit{e.g.} \cite{BCG1} Lemma 5.2) that when $\beta=1$,
$$ \lim_{n \longrightarrow + \infty} n^{-b} \pi_n^{-1} = \Gamma(1-b),$$
whereas when $\beta\neq 1$ 
$$
\lim_{n \longrightarrow + \infty} \exp( b (1-\beta)^{-1} n^{1-\beta}) \pi_n^{-1} = \exp(\Lambda) 
$$
where $\Lambda$ can be made explicit in terms of the Riemann zeta function.
We then conclude the following behaviour of $(S_n)_{n \ge 1}$ (see Appendix B  of \cite{GPS2018}) that:
$$
S_n \sim c_{b,\beta}^{-1} n^{\beta \wedge 1},
$$
which implies that the joint evolution   shall be written as:
\begin{equation}\label{eq:reparam_3}
\begin{cases}
\ttnp = \ttn -  \alpha_{n+1} n^{-\frac{\beta\wedge 1}{2}} \frac{g_{n+1}}{\sqrt{\twn+\te_n}}\\
\twnp=\twn+c_{b,\beta} n^{-(\beta\wedge 1)}[\oo{g_{n+1}} - \twn]
\end{cases}.
\end{equation}
which entails $\gamma_{n+1}=\alpha_{n+1} n^{-\frac{\beta\wedge 1}{2}}$ and $\widetilde{\gamma}_{n+1}=n^{-(\beta\wedge 1)}$.

In all the situations above, we point out that we obtain some standard choices for the sequences $(\gamma_{n+1})_{n \ge 1}$ and $(\widetilde{\gamma}_{n+1})_{n \ge 1}$ involved in our two-time scale system \eqref{eq:two_time_scales}.

\subsubsection{Final remark on step-size sequences\label{sec:final_remark}}

We emphasize that when $(\widetilde{\alpha}_{n})_{n \ge 1}$ is chosen as a constant sequence $\tilde{\alpha}$, the sequence $(\ttn)_{n \ge 1}$ evolves as an ergodic Markov chain and therefore the trajectory cannot converge towards a minimizer of $f$ (indeed it cannot converge anywhere).
Nevertheless, using a finite time horizon strategy with a small enough value of $\tilde{\alpha}$, \cite{Adam_Zou,Adam_Bach} derive some  theoretical guarantees on $\PE[\|\nabla f(\theta_{n})\|^2]$.

\textit{In this work, we have chosen to restrict our study to a single time-scale parametrization} with decreasing sequences $(\gamma_{n+1})_{n \ge 1} =( \widetilde{\gamma}_{n+1})_{n \ge 1}$ within the standard setup of stochastic algorithms:
$$
\sum_{n \ge 1} \gamma_{n+1}= +\infty \qquad \text{and} \qquad \sum_{n \ge 1} \gamma_{n+1}^2< +\infty.
$$
This single time-scale restriction implies that $\sqrt{S_n} = \alpha_{n+1}$, so that when we choose in \eqref{def:adam} $p_n=q_n=1$ and $\gamma_{n+1} = \gamma_1 (n+1)^{-\beta}$, our algorithm is strictly equivalent to the one initially introduced in \eqref{def:adam_bach} with $S_n \propto n^{-\beta}$ and $\alpha_{n+1}=n^{-\beta/2}$. In particular, if $\beta<1$, it corresponds to $\beta_2(n) = 1-b n^{-\beta}$ while if $\beta=1$, it corresponds to $\beta_2(n)=1$.


We leave the more sophisticated general study of the two time-scale algorithm for future investiations and refer to \cite{borkar}, \cite{MokkademPelletier2006} or \cite{BCG1,BCG2}  for other examples of such two time-scale stochastic algorithms in various (but simpler) situations.

\subsection{Assumptions and convenient notations}
We introduce  the canonical filtration associated to our random sequence $\Fn = \sigma\left((\tk,\wk)_{ 1 \leq k \leq n}\right)$ and list below the main assumptions used in our work.

We use the symbols $\led,\gtd$ to refer to inequalities up to a multiplicative constant that are independent from the dimension $d$: for two positive sequences $(u_n)_{n\ge 0}$ and $(v_n)_{n\ge 0}$, we write $$u_n\led v_n \  \mbox{if there exists } C>0 \mbox{ such that} \ u_n\le C v_n, \ \forall n\in\mathbb{N},$$
and the constant $C$ is independent from the dimension of the ambient space $d$. 
We will also use the notation $u_n = \mathcal{O}_d(v_n)$ when $u_n \led v_n$.
\textcolor{black}{
We also use the symbol $\lesssim$ that refers to an inequality up to a multiplicative constant that can depend on $d$, for the proof of the local trap avoidance since for this result we are not interested in a quantitative effect of the dimension.}


\paragraph{Assumptions on the noise.}
We first describe our main assumption on the sequence $(g_n)_{n \ge 1}$.

\noindent
$\bullet$ \underline{Assumption $\mathbf{H}_{\sigma}^p$.}
We assume that the sequence $(g_n)_{n \ge 1}$ used in \eqref{def:adam} provides an unbiased estimation of the true gradient of $f$ at position $\theta_n$, \textit{i.e.} we assume that:
$$
\mathbb{E}[g_{n+1}\, \vert \Fn] = \nabla f(\theta_n).
$$
We furthermore assume that the noise sequence $(\xi_{n+1})_{n \ge 1}$ satisfies:
\begin{equation}\label{Hp}
\forall n \ge 1 \qquad 
\xi_{n+1} := g_{n+1}-\nabla f(\theta_n)= \sigma_{n+1} \zeta_{n+1} \quad \text{with} \quad	\PE[\|\zeta_{n+1}\|^p | \mathcal{F}_n,]\le c (\textcolor{black}{d}+f(\theta_{n}))^{p/2}, 
	\end{equation}
	where $c$ is a positive constant independent from $d$.
Assumption $\mathbf{H}_{\sigma}^p$ stands for a classical framework in stochastic optimization methods: $(\sigma_n)_{n \ge 1}$ is an auxiliary sequence that translastes a possible use of mini-batches when $\sigma_n \longrightarrow 0$ as $n \longrightarrow + \infty$. The moment assumption on $(\zeta_n)_{n \ge 1}$ is the convenient assumption to handle standard problems like on-line regression, logistic regression or cascade of logistic regressions used in deep learning. We emphasize that we do not make any restrictive and somewhat irrealistic boundedness assumption of the noise $(\zeta_n)_{n \ge 1}$ or of the sequence $(\theta_n)_{n \ge 1}$ itself. Below, we will use this assumption with $p=4$ in Theorem \ref{thm: cv ps}. Finally, we should observe that this assumption introduces a possible linear dependency with $d$ on the size of the variance of the noise.
\\ 

To derive the convergence of our algorithms towards a local minima, we will need a more stringent condition on the noise sequence. We then introduce the next assumption that will replace $\mathbf{H}_{\sigma}^p$ in our second main result of almost sure convergence (see Theorem \ref{thm:conv min} below).

\noindent
$\bullet$ \underline{Assumption $\mathbf{H}_{\sigma}^\infty$.} 

\begin{itemize}
\item \underline{$(\mathbf{H}_{\sigma}^\infty-1)$.}
The noise sequence $(\xi_{n+1})_{n \ge 1}$ is centered and satisfies:
\begin{equation}\label{Hinfty}
\xi_{n+1} = \sigma_{n+1} \zeta_{n+1} \quad \text{with} \quad 	\PE[\|\zeta_{n+1}\|^2 | \mathcal{F}_n]\le 1\quad \text{and} \quad 	\PE[\|\zeta_{n+1}\|^4 | \mathcal{F}_n]\le C.
	\end{equation}
\item \underline{$(\mathbf{H}_{\sigma}^\infty-2)$.} The noise sequence is elliptic uniformly in $n$:
$$
\exists m >0 \qquad \forall n \ge 1 \quad \forall u \in \mathcal{S}^{d-1} \qquad \mathbb{E}[\langle u,\zeta_{n+1}\rangle^2] \ge m >0.
$$
\end{itemize}	
We stress that the upper bound $1$ on the second order moment is not restrictive, up to a modification of the calibration of the sequence $(\sigma_{n})_{n \ge 1}$. The second assumption will be of course used to exit local traps.

\paragraph{Assumptions $\Hf$} 
We now introduce some standard assumptions on $f$.

\begin{itemize}
\item \underline{$(\Hf-1)$.} The function $f$ is positive and coercive, \textit{i.e.} $f$ satisfies:
$$
\lim_{\|x\|\longrightarrow + \infty} f(x) = + \infty \qquad \text{and} \qquad \min(f) >0.
$$
Demanding the lower bound of $f$ to be strictly positive is mostly a convenient technical constraint and not fundamentally more restrictive than the classical assumption of positivity.
\item \underline{$(\Hf-2)$.}
	We assume that $f$ satisfies the so-called Lipschitz continuous gradient   property:
	$$
	\exists L>0 \quad \forall (x,y) \in \rset^d \qquad \|\nabla f(x)-\nabla f(y)\| \leq L \|x-y\|.
	$$
We emphasize that this implies the famous descent inequality:
\begin{equation}\label{eq:descent}
f(x) +  \langle h,\nabla f(x)\rangle - \frac{L}{2} \|h\|^2 \leq
f(x+h) \leq f(x) +  \langle h,\nabla f(x)\rangle + \frac{L}{2} \|h\|^2.
\end{equation}
This assumption is commonly used in optimization theory and statistics. Even though it is possible to address some more sophisticated situations (see \textit{e.g.} \cite{not_lip}), it is generally admitted that most of machine learning optimization problems fall into the Lipschitz continuous gradient framework.

\item \underline{$(\Hf-3)$.} We also assume that another  constant $c_f$ exists such that:
	\begin{equation}\label{eq:comparaison}
	|\nabla f|^2 \leq c_f f.
	\end{equation}
This last assumption prevents from some too large growth of the function $f$ and it is immediate to verify that $(\Hf-3)$ implies that $f$ has a subquadratic growth, \textit{i.e.} 
$\lim\sup_{\|x\|\longrightarrow + \infty} \frac{f(x)}{\|x\|^{2}}<+ \infty$.  It has been widely used in the literature of stochastic algorithm (see \textit{e.g.} \cite{GPS2018} and the references therein).
\item \underline{$(\Hf-4)$.}  Finally, we assume that
 $\forall \ x\in \mathbb{R}^d$, $\{ \theta, f(\theta)=x\}\cap\{ \theta,\nabla f (\theta)=0\}$ is locally finite. 

\end{itemize}
%

\paragraph{Assumption on the step-size sequences $\Hg$}
\textcolor{black}{We finally introduce our assumptions on the step-size sequences used all along the paper that are involved in $(p_n)_{n \ge 1}, (q_n)_{n \ge 1}$ and $(\gamma_{n})_{n \ge 1}$.
To easily assess some convergence results with quantitative conditions on our gain sequences, we will consider the situations where:}

\begin{itemize}

\item \underline{$(\Hg-1)$.} The sequences $(p_n)_{n \ge 1}$ and $(q_n)_{n \ge 1}$ satisfy:
$$
\exists (r,p_{\infty}) \in \mathbb{R}^+\times \mathbb{R}^+: \quad |p_n-p_{\infty}| \led n^{-r} \qquad \text{and} \qquad  \lim_{n \longrightarrow + \infty} q_n = q_{\infty} >0.
$$
and
$$
\forall n \ge 1 \qquad \gamma_{n+1} q_n < 1 \qquad \text{and} \lim_{n \longrightarrow
+ \infty} \frac{\gamma_{n+1}}{p_n} = 0.
$$

\item \underline{$(\Hg-2)$.} The mini-batch sequence $(\sigma_n)_{n \ge 1}$ satisfies:
$$
\sigma_n = \sigma_1 n^{-s} \qquad \text{with} \qquad s \ge 0.
$$
\item \underline{$(\Hg-3)$.}  As already discussed in Section \ref{sec:final_remark}, 
the sequence $(\gamma_{n})_{n \ge 0}$ satisfies:
$$
\sum_{n \ge 1} \gamma_{n+1} = + \infty \quad \text{and} \quad \sum_{n \ge 1} \gamma_{n+1}^2 < + \infty.
$$
All the more, we assume that:
$$
 \sum_{n \ge 0} p_n \gamma_{n+1}  \sigma_{n+1}^2 < + \infty.
$$
\end{itemize}
We point out that   $(p_n)_{n \ge 1}$ and $(\sigma_n)_{n \ge 1}$ may be (or not) some vanishing sequences (if $r>0$ and $p_{\infty}=0$ or if $s >0$).

\subsection{Main results}
We now state our three main convergence results for the stochastic algorithm defined in Equation \eqref{def:adam}.

\paragraph{Almost sure convergence result towards a critical point}
\begin{theorem}\label{thm: cv ps} Assume that $\Hf$,
$\mathbf{H}_{\emph{Steps}}$
and $\mathbf{H}_{\sigma}^p$ hold for $p=4$.
Then $(\tn,\wn)_{n \ge 1}$ converges almost surely towards $(\theta_{\infty},0)$ where $\nabla f(\theta_{\infty})=0$.
\end{theorem}
Theorem \ref{thm: cv ps} is a purely asymptotic convergence results. It provides the convergence of our adaptive algorithm \eqref{def:adam} towards a set of \textit{critical points} under mild assumptions on the noise sequence and on the function $f$. We emphasize that this results holds for a standard setup on stochastic algorithms with a decreasing learning rate $(\gamma_{n})_{n \ge 1}$. We observe that the essential condition involved in this result is the convergence of the series that depend on $(\gamma_{n},p_n,\sigma_n^2)$. In particular, when $\gamma_n = \gamma_1 n^{-\beta}$, we observe that Theorem \ref{thm: cv ps} holds when:
$$
\beta \in (1/2,1] \quad \text{and} \quad \beta+r+2s > 1.
$$
From a theoretical point of view, the less restrictive situation corresponds to the choice $\beta=1$ since the series converges as soon as $\sigma_{n+1}^2 p_n$ decreases like $\log(n)^{-2}$. It implies that either we need to use a very lengthy decrease of the update induced by $(p_n)_{n \ge 1}$, or use a very lengthy increase of the minibatch proportional, with a batch of size $\log^2(n)$ at step $n$.
Of course, this last condition holds as soon as $r+2s>0$.
When $\beta$ is chosen lower than $1$, the condition becomes $r+2s>1-\beta$, which may lead to a larger computational cost.

\paragraph{Rate of ``convergence''}

Using the point of view introduced in \cite{GhadimiLan} to assess the computational cost of non-convex stochastic optimization, it is possible to derive a more quantitative result on the sequence $(\theta_n)_{n \ge 1}$.
This result is stated in terms of the expected value of the gradient of $f$ all along the algorithm. A $\delta$-approximation computational cost is then the number of samples that are necessary to obtain an average value below $\delta$. 
\begin{theorem}\label{thm: cv quanti}
Assume that $\Hf$ and $\mathbf{H}_{\sigma}^p$ hold for $p=4$ and consider an integer $N>0$ and $\tau$ an integer sampled uniformly over $\{1,\ldots,N\}$: \begin{itemize}
\item[$i)$] If $\gamma_{n}= \gamma = \frac{1}{d \sqrt{N}}$ and $p_{n}=q_n=\frac{1}{\sqrt{N}}$ and $\sigma_n^2=1$, then
		$$
		\PE\left[ \left\|\sqrt{|\nabla f(\theta_\tau)|}\right\|^4\right] = \mathcal{O}\left(dN^{-1/2}\right)
		$$ 
		and the computational cost to obtain a $\delta$-approximation is $d^2\delta^{-2}$.

		\item[$ii)$] 	If   $\gamma_{n}= \gamma = \frac{1}{ \sqrt{N}}$  and $p_{n}=q_n=1$ and $\sigma_{n}^2 =\frac{1}{d \sqrt{N}}$, then
		$$
		\PE\left[ \left\|\sqrt{|\nabla f(\theta_\tau)|}\right\|^4\right] = \mathcal{O}\left(N^{-1/2}\right)
		$$
		and the computational cost to obtain a $\delta$-approximation is of order $d \delta^{-3}$.
		\item[$iii)$] If  $\gamma_{n}= \gamma = \frac{1}{ \sqrt{ d N}}$ and $p_{n}=q_n=\frac{1}{\sqrt{d N}}$ and $\sigma_n^2=1$, then
		$$
		\PE\left[ \left\|\sqrt{|\nabla f(\theta_\tau)|}\right\|^4\right] = \mathcal{O}\left(dN^{-1/2}\right)
		$$ 
		and the computational cost to obtain a $\delta$-approximation is of order $d\delta^{-2}$.
		\item[$iv)$] 	If   $\gamma_{n}= \gamma = \frac{1}{ \sqrt{N}}$  and $p_{n}=q_n=\frac{1}{ \sqrt{N}}$ and $\sigma_{n}^2 =\frac{1}{d }$, then
		$$
		\PE\left[ \left\|\sqrt{|\nabla f(\theta_\tau)|}\right\|^4\right] = \mathcal{O}\left(N^{-1/2}\right)
		$$ 
		and the computational cost to obtain a $\delta$-approximation is of order $d \delta^{-2}$.
		\end{itemize}
\end{theorem}
 We emphasize that this last result is not a real convergence result, which is indeed impossible to derive with a constant step-size stochastic algorithm. Nevertheless, it may be seen as a benchmark result following the usages in non-convex machine learning optimization.
It is a convenient way to assess a mean square convergence of stochastic optimization algorithm with non-convex landscape (see \textit{e.g.} \cite{GhadimiLan}).

We recover in this result a more quantitative result that translates both the linear effect of the dimension on the ``convergence'' rate and the dependency of the final bound in terms of $N^{-1/2}$ when the algorithm is randomly stopped uniformly between iteration $1$ and $N$. The presence of both $d$ and of $N^{-1/2}$ is not surprising as it already appears to be the minimax rate of convergence in stochastic optimization with weakly convex landscapes (see \textit{e.g.} \cite{Nemirovski_Yudin83}).

If we translate the upper bound of $i)$ into a complexity bound, we observe that for any $\delta>0$, we need to fix $N$ such that
$$
d N^{-1/2} \leq \delta \Longleftrightarrow N \ge (d \delta^{-1})^2.
$$
When we use a mini-batch strategy with $d \sqrt{N}$ samples at each iteration such that the error bound produced by $iv)$ is lower than $\delta$, we observe that $N$ has to be chosen of the order $\delta^{-2}$ and the overall procedure  may be improved (when compared to the first setting) since we obtain a $d \delta^{-3}$ computational cost.
Finally, the otimal tuning of the algorithm seems to be the last ones, where $\sigma^2$ is chosen of the order $d^{-1}$ and $\gamma_n \propto p_n \propto N^{-1}$, or $\sigma^2=1$ and $\gamma=p=q=(d N)^{-1/2}$,
which leads to a $d \delta^{-2}$ computational cost.
As discussed in Section \ref{sec:proof_thm2}, with this strategy, it seems impossible to improve the $d \delta^{-2}$ computational cost obtained with other choices of the parameters.

Even if of rather minor importance, our result is stated with the help of 
$\PE\left[\left\|\sqrt{|\nabla f(\theta_\tau)|}\right\|^4\right]$, instead of $\PE\left[\|\nabla f(\theta_\tau)\|^2\right]$ used in \cite{Adagrad_Ward,Adam_Bach} and $\PE\left[\|\nabla f(\theta_\tau)\|^{4/3}\right]^{2/3}$ used in \cite{Adam_Zou}. It is therefore slightly stronger since (using our vectorial notations):
$$ \left\|\sqrt{|\nabla f|}\right\|^4  \ge \|\nabla f\|^2.$$
A such improvement comes from a careful tuning of a Lyapunov function that is not exactly the same as the one used in these previous works. We refer to Sections \ref{sec:debut_lyap} and \ref{sec:proof_thm2} for further details.
Finally, we also point out that when the sequence $(\gamma_n)_{n \ge 1}$ is kept fixed, as indicated in the paragraph \ref{sec:two_time_scale} 
it corresponds to a choice of $\beta_2< 1$ kept constant all over the time evolution (see Equation \eqref{eq:reparam_2}) and $p = q = \frac{1}{\sqrt{N}}$. This result with this range of parameters appears to be in line with those of \cite{Adagrad_Ward,Adam_Bach}, but in our work the assumptions on the noise sequence and on the function $f$ are significantly weaker.

\paragraph{Almost sure convergence towards a \textit{minimizer}}
In this paragraph, we assess the almost sure convergence of the sequence $(\theta_n)_{n \ge 1}$  towards a local minimum of $f$ and state that the algorithm cannot converge towards an unstable (hyperbolic) point of the dynamical system, \textit{e.g.} cannot converge towards a saddle point or a local maximum of $f$.

\begin{theorem}\label{thm:conv min}
Assume that $\Hf$, $\mathbf{H}_{\emph{Steps}}$ and $\mathbf{H}_{\sigma}^\infty$ hold. Assume that  $f$ is twice differentiable. Suppose $p_\infty=0$, $|q_n-q_{\infty}|= \mathcal{O}(p_n)$, $\gamma_{n}=\gamma_1^{-\beta}$ and let $(\beta,r,s)$ be chosen such that:
$$\frac{1}{2} < \beta < 1  \quad \text{and} \quad  s \leq \frac{1-\beta}{2}\quad \text{and} \quad \left(1-\beta\right) \vee \left(\frac{\beta}{2} +s\right) < r <\beta.$$ Then almost surely the sequence $(\theta_{n})_{n\ge 1}$  does not converge towards a local maximum of $f$. 
\end{theorem}

Several remarks are necessary after our last theorem, that identifies not only the limit points as the critical points of $f$ but as local minimizers. Hence, our contribution should be understood as a new example of stochastic method that avoids local traps, and then compared to \cite{Pemantle,Brandiere,GPS2018,BarakatBianchi2}. 

Moreover, we point out that our result holds for every initialization point  and  do not use any integration over $(\theta_0,w_0)$.
Hence, the nature of our result is different from the ones obtained in recent contributions in the field of machine learning (see \textit{e.g.} \cite{Lee2016,Lee2017}): we establish that our \textit{stochastic algorithm} converges with probability 1 to a minimizer, which is different from proving that a \textit{deterministic or randomized algorithm} (gradient descent in \cite{Lee2017} for example) randomly initialized converges to a local minimum with probability 1. 
\begin{figure}
\begin{center}
\begin{tikzpicture}[scale=4.5]
\draw[->] (-0.25,0) -- (1.25,0);
\draw (1.25,0) node[right] {$s$};
\draw [->] (0,-0.25) -- (0,4/5);
\draw (0,4/5) node[above] {$r$};
\draw [dashed] (1,0.025)-- (1,0) node[below] {$1$};
\fill[color=gray!20] 
 (0,5/12)  -- (5/24,0)
-- (1.2,0)
-- (1.2,7/12)
-- (0,7/12) -- cycle;

\fill[color=red!20] 
 (0,5/12)  -- (3/24,5/12)
-- (5/24,1/2)
-- (5/24,7/12)
-- (0,7/12) -- cycle;

\draw [dashed] (1,0.015)-- (1,0) node[below] {$1$};
\draw [dashed] (7/12,0.015)-- (7/12,0) node[below] {$\beta$};
\draw [dashed] (1/2,0.015)-- (1/2,0) node[below] {$\frac{1}{2}$};
\draw [dashed] (0.025,5/12)-- (0,5/12) node[left] {$1-\beta$};
\draw [dashed] (0.025,7/12)-- (0,7/12) node[left] {$\beta$};
\draw [dashed] (3/24+0.01,0.015)-- (3/24+0.01,0) node[below] {$1-\frac{3\beta}{2}$};
\draw [dashed] (5/24+0.01,0.015)-- (5/24+0.01,0) node[above] {$\frac{1-\beta}{2}$};
\draw (-0.02,-0.0) node[left,below] {$0$};

\node (A) at (0.7,0.85)[color=red]{Local minimizers};
\node (C) at (0.6,1/3){Critical points};

\draw[<-,color=red] (0.05,0.5) -- (A) ;
\end{tikzpicture}
\begin{tikzpicture}[scale=4.5]
\draw[->] (-0.25,0) -- (1.25,0);
\draw (1.25,0) node[right] {$s$};
\draw [->] (0,-0.25) -- (0,4/5);
\draw (0,4/5) node[above] {$r$};
\fill[color=gray!20] 
 (0,1/4)  -- (1/8,0)
-- (1.2,0)
-- (1.2,3/4)
-- (0,3/4) -- cycle;

\fill[color=red!20] 
 (0,3/8)  -- (1/8,1/2)
-- (1/8,3/4)
-- (0,3/4) -- cycle;

\draw [dashed] (1,0.015)-- (1,0) node[below] {$1$};
\draw [dashed] (3/4,0.015)-- (3/4,0) node[below] {$\beta$};
\draw [dashed] (1/2,0.015)-- (1/2,0) node[below] {$\frac{1}{2}$};
\draw [dashed] (1/8,0.015)-- (1/8,0) node[below] {$\frac{1-\beta}{2}$};
\draw (-0.02,-0.0) node[left,below] {$0$};
\draw [dashed] (0,1/4) -- (1,1/4);
\draw [dashed] (0,3/4) -- (1,3/4);
\draw (0,1/4) node[left] {$1-\beta$} [dashed] (0,1/4)-- (-0.025,1/4);
\draw (0,3/8) node[left] {$\beta/2$} [dashed] (0,3/8)-- (-0.025,3/8);
\draw (0,3/4) node[left] {$\beta$} [dashed] (0,3/4)-- (-0.025,3/4);

\node (A) at (0.7,0.85)[color=red]{Local minimizers};
\node (C) at (0.6,1/3){Critical points};

\draw[<-,color=red] (0.05,0.5) -- (A) ;
\end{tikzpicture}
\end{center}
\caption{Nature of our convergence results when $(s,r)$ are chosen according to the statements of Theorem \ref{thm: cv ps} and Theorem \ref{thm:conv min} when 
$\gamma_{n} = \gamma_1 n^{-\beta}$ and $\beta \ge 1/2$. Left: $\beta \le 2/3$. Right: $\beta \ge 2/3$. The mini-batch size is $\sigma_n^{-2} \propto n^{2s}$.}
\end{figure}
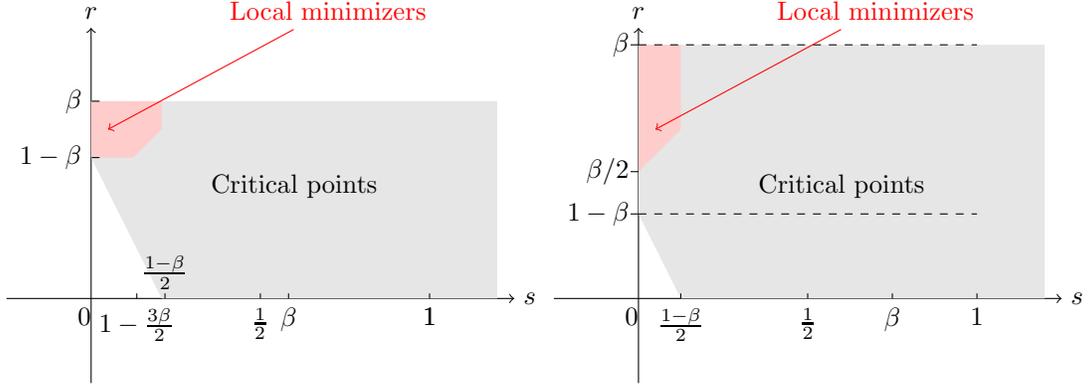

	When the variance of the noise sequence is kept fixed all along the iterations (no mini-batch is used, so that $s=0$), the previous conditions on the parameters can be summarized as: $\beta \in (1/2,1)$; $r\in (1-\beta,\beta)$, when $\beta<2/3$ and $r\in (\beta/2,\beta)$ when $\beta\ge 2/3$. 

Finally, we should emphasize that this last result is rather difficult to obtain when the mini-batch parameter $s$ is chosen strictly greater than $0$ since it translates a possible vanishing level of noise when the number of iterations is increasing. Our assumption shows that the size of the mini-batch should not grow to fast (induced by the condition $s \leq (1-\beta)/2$) to obtain the convergence towards a local minimizer of $f$. Up to our knowledge, a such explicit phenomenon is new in the stochastic algorithm community. It would deserve further numerical or theoretical  investigations to identify whether this condition is necessary to convert an almost sure convergence result towards critical points into a convergence result towards a stable point of the differential system. The limiting condition seems to be
$\beta=1, s \le 1/2$ and $r \in [1/2,1]$. As a really ambitious question, we leave this problem for future developments and up to our knowledge, the maximal size of the mini-batch that guarantees convergence to local minimizer is still an unresolved question even for the SGD.

\subsection{Organization of the paper}
The rest of the paper consists in showing the proofs of the previous results.

Theorems \ref{thm: cv ps} and \ref{thm: cv quanti} are proven in Section \ref{sec:as}. In particular, Proposition \ref{prop:rec_termes} studies the average one-step evolution of the algorithm through several key functions. This proposition permits to derive a Lyapunov function in Section \ref{sec:proof_thm2} that translates both the average quantitative result of Theorem \ref{thm: cv quanti} and the asymptotic convergence result of Theorem \ref{thm: cv ps}. The main difficulties in these two results is to derive a mean reverting effect in terms of $\|\nabla f(\theta_n)\|^2$ without using some extra boundedness assumption, and to assess the influence of $d$ on the quantitative result.\\

Theorem \ref{thm:conv min} is a typical result of stochastic algorithms, and is inspired from the seminal contributions of \cite{Pemantle} and \cite{benaim1995dynamics}. The proof is detailed in Section \ref{sec:instable} and the cornerstone of this proof is the use of the stable/unstable manifold Lemma that provides an ad-hoc Lyapunov function of the dynamical system, denoted by $\eta$ in Proposition \ref{prop:instable}. We also refer to the recent contribution of \cite{BarakatBianchi2} for another typical application to stochastic algorithms. The main novelty brought in our  proof is to use the mini-batch low noise level and keep the a.s. escape of local maximum true. In particular, from a technical point of view, we take advantage of the boundedness series of Proposition \ref{prop:estimation_series}, which is a key ingredient in the proof of Proposition \ref{prop:S_infini}.

\section{Almost sure convergence to the set of critical points\label{sec:as}}

The purpose of this section is to prove Theorem \ref{thm: cv ps} and Theorem \ref{thm: cv quanti}. In particular, we will obtain Theorem \ref{thm: cv quanti} 
 during the proof of the almost sure convergence result as a specific point of Proposition \ref{prop:estimation_series}, $iii)$. The basic ingredient of our proof relies on the Robbins-Siegmund Theorem \cite{robbins1971convergence} that will be applied with the help of an ad-hoc Lyapunov function on $(\theta_n,w_n)_{n \ge 1}$.

\subsection{Preliminary computations\label{sec:debut_lyap}}
 Below, we will pay a specific attention to the dimension dependency in the inequalities we will obtain.
We first state the following proposition that will create a mean reverting effect from iteration $n$ to iteration $n+1$ on the pair $(\tn,\wn)$.

\begin{proposition}\label{prop:rec_termes}
	Assume that $ \|\gamma_n q_n\|_{\infty} < 1$, that $\|p_n\|_{\infty}<+\infty $,  that $\Hf$ and $\mathbf{H}_{\sigma}^p$ hold for $p=4$.
	\begin{itemize}
		\item[$i)$] For any $n \ge 1$, one has:
		\begin{align*}		
		\PE[\|\sqrt{w_{n+1}+\varepsilon}\|^2|\mathcal{F}_n]&\le  \|\sqrt{w_{n}+\varepsilon}\|^2 -q_n \gamma_{n+1} \|\sqrt{w_{n}}\|^2  \\
		& + \gamma_{n+1} p_n \| \nabla f(\theta_n) \|^2  + \gamma_{n+1} \sigma_{n+1}^2p_n(\textcolor{black}{d}+f(\theta_{n})).
		\end{align*}
		\item[$ii)$] A constant $\cu$ independent from $n$ and $\textcolor{black}{d}$ exists such that for any $n \ge 1$:
		
		$$
		\mathbb{E}[ f(\tnp) \, \vert| \mathcal{F}_n] \leq f(\tn) (1+ \cu \gamma_{n+1}^2)+ \cu \textcolor{black}{d} \gamma_{n+1}^2 \sigma_{n+1}^2-  \gamma_{n+1} \left|\left|\dfrac{\nabla f (\theta_n)}{(w_n+\varepsilon)^{1/4}}\right|\right|^2.
		$$
		\item[$iii)$]
%
		If we define 
		$$s_n:= \gamma_{n+1}^2(1+ \sigma_{n+1}^2d+\gamma_{n+1}^2\sigma_{n+1}^4d^2),$$	
		then a constant $\cd$ independent from $n$ and $d$ exists such that:
		\begin{align*}
		\PE[f(\tn)^2| \F_n]&\le f(\tn)^2(1+\cd s_n)- \textcolor{black}{2}\gamma_{n+1} f(\tn)\left\|\dfrac{\nabla f (\theta_n)}{(w_n+\varepsilon)^{1/4}}\right\|^2 +
		\cd s_n. 
		\end{align*}
		\item[$iv)$] 
		If we define 
		
		\begin{equation*}\label{def:t_n}
		t_n:=\gamma_{n+1}p_n(\gamma_{n+1}^2+2\sigma_{n+1}^2d+2\gamma_{n+1}^2\sigma_{n+1}^4d^2)
				\end{equation*}

%
		
		then a constant $\ct$ independent from $n$ and $d$ exists such that:
		
		
			\begin{align*}
		\PE&[f(\tnp)  \|(\wnp+\varepsilon)^{1/4}\|^2 \, \vert \Fn] \leq 
		f(\tn)  \|(\wn+\varepsilon)^{1/4}\|^2 \left(1+\ct \gamma_{n+1}^2 (1+d\sigma_{n+1}^2)\right) \\
		& - \gamma_{n+1} 	\left\|\sqrt{|\nabla f(\tn)|}\right\|^4 +  \gamma_{n+1} p_n f(\tn)\left\|\frac{\nabla f(\tn)}{(w_n+\varepsilon)^{1/4}}\right\|^2 \\
		& + \ct\left[ t_n f^2(\tn)+t_n \large\right].
		\end{align*}
	\end{itemize}
\end{proposition}

\begin{proof} We consider each point separately.
	
	\noindent $\bullet$ \underline{\textit{Proof of $i)$:}}
	We observe that:
	\begin{align*}
	\mathbb{E}[\|\sqrt{w_{n+1}+\varepsilon}\|^2 \, \vert| \mathcal{F}_n] &= 
	\mathbb{E}\left[\left\|\sqrt{w_n+\gamma_{n+1}[\oo{p_n g_{n+1}}-q_n w_n]+\varepsilon}\right\|^2 \, \vert \Fn\right]\\
	& = \PE \left[ \sum_{i=1}^d w_n^i+\varepsilon + \gamma_{n+1}[p_n \{g_{n+1}^i\}^2-q_n w_n^i]\, \vert \Fn \right] \\
	& \leq   \|\sqrt{w_{n}+\varepsilon}\|^2 - q_n \gamma_{n+1} \|\sqrt{w_{n}}\|^2 + \gamma_{n+1} p_n
	\| \nabla f(\theta_n) \|^2 + \gamma_{n+1} p_n \sigma_{n+1}^2(d+f(\theta_{n})),
	\end{align*}
	where the last line comes from the definition of $\xi_{n+1} = \sigma_{n+1} \zeta_{n+1}$ and the fact that:
	$$\PE[\zeta_{n+1} \, \vert  \Fn] = 0 \quad \text{and} \quad \PE[\|\zeta_{n+1}\|^2 \, \vert  \Fn] \led (\textcolor{black}{d}+f(\theta_{n})).$$
	This concludes the proof. \hfill $\diamond$
	
	\noindent $\bullet$ \underline{\textit{Proof of $ii)$:}}
	We develop $f(\tnp)$: we use the descent inequality \eqref{eq:descent}: 
	\begin{align}
	f&(\tnp)  = f\left(\tn - \gamma_{n+1}\frac{g_{n+1}}{\sqrt{w_n+\varepsilon}}\right) \nonumber\\
	& \leq f(\tn) - \gamma_{n+1} \langle \nabla f(\tn),\frac{g_{n+1}}{\sqrt{w_n+\varepsilon}} \rangle
	+ \frac{L^2}{2}\gamma_{n+1}^2 \left\| \frac{g_{n+1}}{\sqrt{w_n+\varepsilon}}\right\|^2 := m^+_n
	\label{eq:majoration}\\
	& \leq f(\tn) -  \gamma_{n+1} \left|\left|\dfrac{\nabla f (\theta_n)}{(w_n+\varepsilon)^{1/4}}\right|\right|^2 - \gamma_{n+1} \langle \nabla f(\tn),\frac{\xi_{n+1}}{\sqrt{w_n+\varepsilon}}\rangle \nonumber\\
	& +  L^2  \varepsilon^{-1} \gamma_{n+1}^2   \left\|\nabla f (\theta_n)\right\|^2 +  L^2 \varepsilon^{-1}  \gamma_{n+1}^2  \sigma_{n+1}^2 
	\left\| \zeta_{n+1}\right\|^2,\nonumber 
	\end{align}
	where in the last line we used a rough upper bound on
	$1/\sqrt{w_n+\varepsilon} $ and $|a+b|^2 \leq 2 a^2+2b^2$.
	It is also possible to use \eqref{eq:descent} and prove the lower bound:
	\begin{align}
	f(\tnp) & \ge f(\tn) - \gamma_{n+1} \langle \nabla f(\tn),\frac{g_{n+1}}{\sqrt{w_n+\varepsilon}} \rangle
	- \frac{L^2}{2}\gamma_{n+1}^2 \left\| \frac{g_{n+1}}{\sqrt{w_n+\varepsilon}}\right\|^2 := m^-_n.
	\label{eq:minoration}
	\end{align}
	
	
	To deduce $ii)$, we then consider the conditional expectation in \eqref{eq:majoration}, use Equation \eqref{eq:comparaison} and assumption $\mathbf{H}_{\sigma}^p$ for $p=2$ to
	obtain that a constant $c_{\epsilon}=L^2(1\vee c_f)/\varepsilon $ exists such that:
	$$
	\mathbb{E}[ f(\tnp) \, \vert| \mathcal{F}_n] \leq f(\tn) (1+ c_\epsilon  \gamma_{n+1}^2)+ c_{\epsilon} \textcolor{black}{d} \gamma_{n+1}^2 \sigma_{n+1}^2 -  \gamma_{n+1} \left|\left|\dfrac{\nabla f (\theta_n)}{(w_n+\varepsilon)^{1/4}}\right|\right|^2.
	$$
	
	We emphasize that at this stage, this last inequality does not create any repelling effect on the position $(\tn)_{n \ge 1}$ and we need to deal with the denominator $(w_n+\varepsilon)$, which is the purpose of $iii)$. \hfill $\diamond$
	
	\noindent $\bullet$ \underline{\textit{Proof of $iii)$:}}
	We  consider the evolution of $f^2$ from $\tn$ to $\tnp$. Using \eqref{eq:majoration} and \eqref{eq:minoration} obtained in $ii)$, we deduce that:
	\begin{align*}
	f(\tnp)^2 & \le \{m^-_n\}^2 \vee \{m^+_n\}^2 .
	\end{align*}
	We now expand $\{m^-_n\}^2$ and  $\{m^+_n\}^2$, using that $a \vee -a = |a|$, we get:
	\begin{align}
	\{m^-_n\}^2& \vee \{m^+_n\}^2  \le f(\tn)^2 + \gamma_{n+1}^2 
	\langle \nabla f(\tn),\frac{g_{n+1}}{\sqrt{w_n+\varepsilon}} \rangle^2+
	\frac{L^4}{4}\gamma_{n+1}^4 \left\| \frac{g_{n+1}}{\sqrt{w_n+\varepsilon}}\right\|^4
	\nonumber\\
	&- 2  \gamma_{n+1} f(\tn) \langle \nabla f(\tn),\frac{g_{n+1}}{\sqrt{w_n+\varepsilon}} \rangle + L^2 \gamma_{n+1}^2 f(\tn) \left\| \frac{g_{n+1}}{\sqrt{w_n+\varepsilon}}\right\|^2\nonumber\\
	& + L^2 \gamma_{n+1}^3  \left\| \frac{g_{n+1}}{\sqrt{w_n+\varepsilon}}\right\|^2 
	\left|  \langle \nabla f(\tn),\frac{g_{n+1}}{\sqrt{w_n+\varepsilon}} \rangle\right| \nonumber\\
	\label{eq:intermediaire_horrible}
	\end{align}
	When taking the conditional expectation in the previous inequality we treat some of these terms separately.	
	First, using the centering of $\xi_{n+1}$ we have that :
	\begin{align*}
	\gamma_{n+1}^2\PE[\langle \nabla f(\tn),\frac{g_{n+1}}{\sqrt{w_n+\varepsilon}} \rangle^2|\mathcal{F}_n]&\le \gamma_{n+1}^2 \left(\left\| \frac{\nabla f(\tn)}{\sqrt{w_n+\varepsilon}}\right\|^4 + \PE[\langle \xi_{n+1},\frac{\nabla f(\tn)}{\sqrt{w_n+\varepsilon}} \rangle^2|\mathcal{F}_n] \right)\\
	&\le\gamma_{n+1}^2 \left(\varepsilon^{-2} \left\| \nabla f(\tn)\right\|^4 + \left\| \frac{\nabla f(\tn)}{\sqrt{w_n+\varepsilon}}\right\|^2\sigma_{n+1}^2(d+f(\theta_{n}))\right)\\
	&\led \gamma_{n+1}^2 (1+\sigma_{n+1}^2d)(1+f(\theta_{n}))^2,
	\end{align*}
	where in the second line we used Cauchy-Schwarz inequality and assumption  $\Hp$ for $p=2$	whereas the last line comes from sub-quadratic growth assumption on the function $f$ given by \eqref{eq:comparaison}. 
	This last bound can also be used to control  the term $L^2 \gamma_{n+1}^2 f(\tn) \left\| \frac{g_{n+1}}{\sqrt{w_n+\varepsilon}}\right\|^2$ since:
	$$\left\| \frac{g_{n+1}}{\sqrt{w_n+\varepsilon}}\right\|^2\led \|\xi_{n+1}\|^2+\|\nabla f(\tn))\|^2.$$
	Now, in similar manner, using repeatitly these last two assumptions ($\Hp$ also with $p=4$ ) :
	\begin{align*}
	\frac{L^4}{4}\gamma_{n+1}^4 \PE \left[ \left\| \frac{g_{n+1}}{\sqrt{w_n+\varepsilon}}\right\|^4| \Fn\right]\le & \frac{L^4}{4\varepsilon^2} \gamma_{n+1}^4 \left( \| \nabla f(\theta_{n})\|^4+6 \| \nabla f(\theta_{n})\|^2\PE[ \|\xi_{n+1}\|^2| \Fn]\right.\\
	&\left. +4\| \nabla f(\theta_{n})\|\PE[ \|\xi_{n+1}\|^3| \Fn] +\PE[ \|\xi_{n+1}\|^4| \Fn] \right) \\
	& \led \gamma_{n+1}^4 (1+\sigma_{n+1}\sqrt{d})^4 (1+f(\theta_{n}))^2\\
	\end{align*}
	Starting again with the Cauchy-Schwarz inequality for the last term we obtain :
	
	\begin{align*}
	L^2 \gamma_{n+1}^3 \PE\left[ \left\| \frac{g_{n+1}}{\sqrt{w_n+\varepsilon}}\right\|^2 
	\left|  \langle \nabla f(\tn),\frac{g_{n+1}}{\sqrt{w_n+\varepsilon}} \rangle\right| | \Fn\right] & \led \gamma_{n+1}^3  \PE\left[ \left\| g_{n+1}\right\|^3 
	\| \nabla f(\tn)\| | \Fn\right]\\
	&\led \gamma_{n+1}^3 \|\nabla f(\tn)\| \PE[( \|\nabla f(\tn)\|+ \|\xi_{n+1}\|)^3|\Fn ]\\
	&	\led \gamma_{n+1}^3  \sqrt{f(\tn)} \PE[( \sqrt{f(\tn)}+ \|\xi_{n+1}\|)^3|\Fn ]	\\
	& \led \gamma_{n+1}^3 (1+f(\theta_{n}))^2(1+\sigma_{n+1}\sqrt{d})^3.
	\end{align*}
	
	We can now regroup the previous bounds, use equation $\eqref{eq:intermediaire_horrible}$, the fact that $\gamma_{n}<1$ and that $\sqrt{ab}\le 1/2(a+b)$ to conclude that a constant $\cd$, independent of $d$ and $n$ exists such that :
	%
	%
	
	\begin{align*}
	\PE[f(\tn)^2| \F_n]&\le f(\tn)^2(1+\cd s_n)- \textcolor{black}{2}\gamma_{n+1} f(\tn)\left\|\dfrac{\nabla f (\theta_n)}{(w_n+\varepsilon)^{1/4}}\right\|^2 +
	\cd s_n. \qquad \diamond 
	\end{align*}
	\noindent $\bullet$ \underline{\textit{Proof of $iv)$:}}
	We observe that:
	\begin{align*}
	\|(\wnp+\varepsilon)^{1/4}\|^2 &= \sum_{i=1}^d \sqrt{\wn^i+\varepsilon +\gamma_{n+1} [p_n \{g_{n+1}^i\}^2-q_n w_n^i]}\\
	& = \sum_{i=1}^d \sqrt{(\wn^i+\varepsilon)\left(1+\gamma_{n+1}\frac{p_n \{g_{n+1}^i\}^2-q_n w_n^i}{\wn^i+\varepsilon}\right)}\\
	&  = \sum_{i=1}^d \sqrt{\wn^i+\varepsilon}\left(1+\gamma_{n+1}\frac{p_n \{g_{n+1}^i\}^2-q_n w_n^i}{\wn^i+\varepsilon}\right)^{1/2},
	\end{align*}
	where the last line comes from the fact that $w_n^{i}/(w_n^i+\varepsilon)<1$, which implies that the last term of the right hand side exists. 
	
	Using $\sqrt{1+a} \leq 1+a/2$ for any $a>-1$, we deduce that:
	\begin{align*}
	\|(\wnp+\varepsilon)^{1/4}\|^2& \leq \sum_{i=1}^d \sqrt{\wn^i+\varepsilon} \left(1+\frac{\gamma_{n+1}}{2} \frac{p_n \{g_{n+1}^i\}^2-q_n w_n^i}{\wn^i+\varepsilon}\right)\\
	& \leq \|(\wn+\varepsilon)^{1/4}\|^2 + \frac{\gamma_{n+1}}{2} \langle \frac{1}{\sqrt{\wn+\varepsilon}},p_n \oo{g_{n+1}}-q_n\wn\rangle.
	\end{align*}
	Now observe that the second term can easily be bounded by : 
	\begin{align*}
	\langle \frac{1}{\sqrt{\wn+\varepsilon}},p_n \oo{g_{n+1}}-q_n\wn\rangle &=  \sum_{i=1}^d  \frac{p_n \{g_{n+1}^i\}^2}{\sqrt{w_n^i+\varepsilon}} - \underbrace{ \langle \frac{1}{\sqrt{\wn+\varepsilon}},q_n \wn\rangle}_{ \text{positive term}}
	\\
	& \leq  \sum_{i=1}^d  \frac{p_n \{g_{n+1}^i\}^2}{\sqrt{w_n^i+\varepsilon}}\\
	&\le 2 p_n  \left(\left\|\frac{\nabla f(\tn)}{(w_n+\varepsilon)^{1/4}}\right\|^2 + \varepsilon^{-\textcolor{black}{1/2}} \sigma_{n+1}^2 \|\zeta_{n+1}\|^2]\right)\\
	\end{align*}
	We use this last inequality and $f(\tnp) \leq m_n^{+}$ to conclude that:
	$$
	f(\tnp) \|(\wnp+\varepsilon)^{1/4}\|^2 \leq m_n^{+} \left(\|(\wn+\varepsilon)^{1/4}\|^2 + \gamma_{n+1}p_n   \left(\left\|\frac{\nabla f(\tn)}{(w_n+\varepsilon)^{1/4}}\right\|^2 + \varepsilon^{-\textcolor{black}{1/2}} \sigma_{n+1}^2 \|\zeta_{n+1}\|^2]\right)\right).
	$$
	We then develop and obtain that:
	\begin{align*}
	f&(\tnp) \|(\wnp+\varepsilon)^{1/4}\|^2  \leq \left( f(\tn) -  \gamma_{n+1} \left|\left|\dfrac{\nabla f (\theta_n)}{(w_n+\varepsilon)^{1/4}}\right|\right|^2 - \gamma_{n+1} \langle \nabla f(\tn),\frac{\xi_{n+1}}{\sqrt{w_n+\varepsilon}}\rangle \right. 
	+  L^2  \varepsilon^{-1} \gamma_{n+1}^2   \left\|\nabla f (\theta_n)\right\|^2  \nonumber\\
	& \left. + L^2 \varepsilon^{-1}  \gamma_{n+1}^2  \sigma_{n+1}^2 \|\zeta_{n+1}\|^2 \right)
	\left(\|(\wn+\varepsilon)^{1/4}\|^2 + \gamma_{n+1}p_n   \left(\left\|\frac{\nabla f(\tn)}{(w_n+\varepsilon)^{1/4}}\right\|^2 + \varepsilon^{\textcolor{black}{-1/2}} \sigma_{n+1}^2 \|\zeta_{n+1}\|^2]\right)\right)
	\end{align*}
	The baseline remark is that thanks to the Cauchy-Schwarz  inequality, we have:
	\begin{align}
	\left\|\sqrt{|\nabla f(\tn)|}\right\|^4 &= \left(\sum_{i=1}^d |\partial_i f(\tn)|\right)^2 \nonumber\\
	&\leq \sum_{i=1}^d \frac{\{\partial_i f(\tn)\}^2}{\sqrt{w_n^i+\varepsilon}} \sum_{i=1}^d \sqrt{w_n^i+\varepsilon}\nonumber\\
	& = \left|\left|\dfrac{\nabla f (\theta_n)}{(w_n+\varepsilon)^{1/4}}\right|\right|^2 \|(\wn+\varepsilon)^{1/4}\|^2. \label{eq:cs_malin}
	\end{align}
	
	We then compute the conditional expectation of the previous terms with respect to $\Fn$, using the previous inequality, the centering of $\xi_{n+1}$, 
	and obtain that a constant $\ct$ independent from $d$ exists such that:
	%

	\begin{align*}
	\PE&[f(\tnp)  \|(\wnp+\varepsilon)^{1/4}\|^2 \, \vert \Fn]\\
	& \leq 
	f(\tn)\|(\wn+\varepsilon)^{1/4}\|^2 + \gamma_{n+1} p_n f(\tn)\left\|\frac{\nabla f(\tn)}{(w_n+\varepsilon)^{1/4}}\right\|^2 - \gamma_{n+1} 	\left\|\sqrt{|\nabla f(\tn)|}\right\|^4 \\
	& + \ct\PE\left(\underbrace{ \gamma_{n+1} \sigma_{n+1}^2  p_n f(\tn) \|\zeta_{n+1}\|^2}_{:=\textcircled{1}}   
	+  \underbrace{ \gamma_{n+1}^2 p_n \sigma_{n+1}^3   \left| \langle \nabla f(\tn),\frac{\zeta_{n+1}}{\sqrt{w_n+\varepsilon}}\rangle \right| \|\zeta_{n+1}\|^2 }_{:=\textcircled{2}}   \right. \\
	&+   \underbrace{\gamma_{n+1}^2 \|\nabla f(\tn)\|^2\|(\wn+\varepsilon)^{1/4}\|^2}_{:=\textcircled{3}} + \underbrace{ \gamma_{n+1}^3\textcolor{black}{p_n}  \|\nabla f(\tn)\|^2\left\|\frac{\nabla f(\tn)}{(w_n+\varepsilon)^{1/4}}\right\|^2}_{:=\textcircled{4}} + \underbrace{ \gamma_{n+1}^{\textcolor{black}{3}}\textcolor{black}{p_n} \sigma_{n+1}^2\|\nabla f(\tn)\|^2 \|\zeta_{n+1}\|^2}_{:=\textcircled{5}} \\
	& \left. +\underbrace{ \gamma_{n+1}^2  \sigma_{n+1}^2  \|\zeta_{n+1}\|^2\|(\wn+\varepsilon)^{1/4}\|^2}_{:=\textcircled{6}}+  \underbrace{ p_n \gamma_{n+1}^3  \sigma_{n+1}^2  \|\zeta_{n+1}\|^2\left\|\frac{\nabla f(\tn)}{(w_n+\varepsilon)^{1/4}}\right\|^2}_{:=\textcircled{7}}  +   \underbrace{\gamma_{n+1}^{\textcolor{black}{3}}\textcolor{black}{p_n} \sigma_{n+1}^4  \|\zeta_{n+1}\|^4 \, \big|}_{:=\textcircled{8}} \Fn \right)
	\end{align*}
	We then study each terms in the large bracket separately. Using $f\le 1+f^2$ and $\Hp$, we have:
	\begin{align*}
	\PE[
	\textcircled{1}\big| \Fn ]&=\gamma_{n+1} \sigma_{n+1}^2  p_n  f(\tn)  \PE\left( \|\zeta_{n+1}\|^2  \big| \Fn \right)\\
	&  \leq \gamma_{n+1} \sigma_{n+1}^2  p_n  f(\tn) (d+f(\tn)) \\
	&  \textcolor{black}{\le \gamma_{n+1} \sigma_{n+1}^2  p_n  d (1+f(\tn))^2}.
	\end{align*}
	The second term is handled with the help of  $\Hp$ (for $p=3$) and the Cauchy-Schwarz inequality:
	\begin{align*}
	\PE[
	\textcircled{2}\big| \Fn ]&= \gamma_{n+1}^2 p_n \sigma_{n+1}^3 \|\nabla f(\tn)\| \PE[\|\zeta_{n+1}\|^3\big| \Fn ]\\
	& \le \gamma_{n+1}^2 p_n \sigma_{n+1}^3 \|\nabla f(\tn)\| (d+f(\tn))^{3/2}\\
	&\textcolor{black}{\le  \gamma_{n+1}^2 p_n \sigma_{n+1}^3 \sqrt{c_f(1+f(\theta_{n}))} d^{3/2} (1+f(\tn))^{3/2}}\\
	& \textcolor{black}{\led  \gamma_{n+1}^2 p_n \sigma_{n+1}^3  d^{3/2} (1+f(\tn))^{2}}.
	\end{align*}
	\textcircled{3} is $\Fn$ measurable and the subquadratic growth assumption given by \eqref{eq:comparaison} ensures that :
	$$
	\textcircled{3} \led \gamma_{n+1}^2f(\tn)\|(\wn+\varepsilon)^{1/4}\|^2.
	$$
%
	The term \textcircled{4} is $\Fn$ measurable and we use once more that $\|\nabla f\|^2 \led f$ to obtain that:
	\begin{align*}
	\textcircled{4}  & \led \gamma_{n+1}^3\textcolor{black}{p_n} f(\tn) \sum_{i=1}^d \frac{\{\partial_i f(\tn)\}^2}{\sqrt{w_n^i+\varepsilon}} \leq \gamma_{n+1}^3\textcolor{black}{p_n} f(\tn)\varepsilon^{-1/2} \sum_{i=1}^d \{\partial_i f(\tn)\}^2
	\\
	& \led \gamma_{n+1}^3\textcolor{black}{p_n} f(\tn)^2.
	\end{align*}
	For \textcircled{5} we use Assumption $\Hp$ with $p=2$ and obtain that:
	\begin{align*}
	\PE[
	\textcircled{5}\big| \Fn ] &= \gamma_{n+1}^{\textcolor{black}{3}}\textcolor{black}{p_n} \sigma_{n+1}^2 \|\nabla f(\tn)\|^2 \PE[ \|\zeta_{n+1}\|^2 \big| \Fn ]  \\
	& \led \gamma_{n+1}^{\textcolor{black}{3}}\textcolor{black}{p_n} \sigma_{n+1}^2 f(\tn) (d+f(\tn))\\
	& \led  \gamma_{n+1}^{\textcolor{black}{3}}\textcolor{black}{p_n} \sigma_{n+1}^2 d (1+f(\tn))^2.
	\end{align*}
	Since $\gamma_{n+1}^2\le 1$, we can conclude that $\PE[
	\textcircled{5}\big| \Fn ]\led \PE[
	\textcircled{1}\big| \Fn ]$.

	\textcircled{6} is close to \textcircled{3},  we use $\Hp$ with $p=2$ and obtain that:

	\begin{align*}
		\PE[
	\textcircled{6}\big| \Fn ]& \leq \gamma_{n+1}^2 \sigma_{n+1}^2 \|(\wn+\varepsilon)^{1/4}\|^2 \PE[\|\zeta_{n+1}\|^2 \, \vert \Fn] \led \gamma_{n+1}^2 \sigma_{n+1}^2 (d+f(\tn)) 
	\|(\wn+\varepsilon)^{1/4}\|^2\\
	&\textcolor{black}{\led d\gamma_{n+1}^2 \sigma_{n+1}^2 f(\tn)
	\|(\wn+\varepsilon)^{1/4}\|^2 }.
	\end{align*}
The last line comes from the fact that as soon as $f$ is uniformly lower bounded by a positive constant, then $\|(\wn+\varepsilon)^{1/4}\|^2 \led f(\tn)\|(\wn+\varepsilon)^{1/4}\|^2$.

	The term \textcircled{7} is close to \textcolor{black}{ \textcircled{5}} and $\Hp$ yields:
	\begin{align*}
	\PE[
	\textcircled{7}\big| \Fn ]& \leq p_n \gamma_{n+1}^3 \sigma_{n+1}^2 \varepsilon^{-1/2} \|\nabla f(\tn)\|^2 \PE[ \|\zeta_{n+1}\|^2\big| \Fn] \\
	& \led p_n \gamma_{n+1}^3 \sigma_{n+1}^2 f(\tn)(d+f(\tn)) \\ 
	& \led \PE[ \textcircled{5}\big| \Fn ].
	\end{align*}
	For \textcircled{8} Assumption $\Hp$ with $p=4$ implies that:
	\begin{align*}
	\PE[
	\textcircled{8}\big| \Fn ] \led\gamma_{n+1}^{\textcolor{black}{3}}\textcolor{black}{p_n} \sigma_{n+1}^4 (d^2+f^2(\tn)).
	\end{align*}
	
	Our bounds on $\textcircled{1}-\textcircled{8}$ and the fact that $(1+f(\tn))^2\le 2(1+f(\tn)^2)$ ensure that a constant $\ct$ exists independent from $d$ such that:
	\begin{align*}
		\PE&[f(\tnp)  \|(\wnp+\varepsilon)^{1/4}\|^2 \, \vert \Fn] \leq 
		f(\tn)  \|(\wn+\varepsilon)^{1/4}\|^2 \left(1+\ct  \gamma_{n+1}^2 (1+d\sigma_{n+1}^2)\right)\\
		& - \gamma_{n+1} 	\left\|\sqrt{|\nabla f(\tn)|}\right\|^4 
		 +  \gamma_{n+1} p_n f(\tn)\left\|\frac{\nabla f(\tn)}{(w_n+\varepsilon)^{1/4}}\right\|^2 \\
		& + \ct\left[ \gamma_{n+1}p_n(\gamma_{n+1}^2+\sigma_{n+1}^2d +\gamma_{n+1}\sigma_{n+1}^3d^{3/2}+\gamma_{n+1}^2\sigma_{n+1}^4d^2)(1+ f^2(\tn))\right].
		\end{align*}
	Since $ \gamma_{n+1}\sigma_{n+1}^3d^{3/2}  \le 1/2(\sigma_{n+1}^2d+\gamma_{n+1}^2\sigma_{n+1}^4d^2),$
  using the definition of $t_n$,  we deduce that:
	
	\begin{align*}
	\PE&[f(\tnp)  \|(\wnp+\varepsilon)^{1/4}\|^2 \, \vert \Fn] \leq 
	f(\tn)  \|(\wn+\varepsilon)^{1/4}\|^2 \left(1+\ct \gamma_{n+1}^2 (1+d\sigma_{n+1}^2)\right) \\
	& - \gamma_{n+1} 	\left\|\sqrt{|\nabla f(\tn)|}\right\|^4 +  \gamma_{n+1} p_n f(\tn)\left\|\frac{\nabla f(\tn)}{(w_n+\varepsilon)^{1/4}}\right\|^2 \\
	& + \ct\left[ t_n f^2(\tn)+t_n \large\right]. \qquad \diamond
	\end{align*}

\end{proof}

\begin{remark} Proposition \ref{prop:rec_termes} will permit do derive a Lyapunov function on $(\tn,\wn)_{n \ge 1}$ (see the next result) which implies the convergence of 
	$$
	\PE\left[ \sum_{n \ge 1} \gamma_{n+1} \left\| \nabla f (\theta_n) \right\|^2\right] < + \infty.
	$$
	This kind of bound has also been obtained in \cite{Adam_Bach} (Theorem 4) with the help of a somewhat artificial boundedness assumption of the noisy gradients, which is not used in our work.
	We also point out that \cite{Adam_Zou} propose another function that generates a mean reverting term:
	$$
	\sum_{n \ge 1} \gamma_{n+1} \PE[ \|\nabla f(\tn)\|^{4/3}] < \infty, 
	$$
	and the major difference with our result is the weaker $4/3$ instead of $2$ in the series. In particular, a such $4/3$ will not allow to prove the a.s. asymptotic pseudo-trajectory result, and consequently the a.s. convergence of the trajectory towards a critical point of $f$.
\end{remark}
%

\subsection{Proof of Theorem \ref{thm: cv quanti}\label{sec:proof_thm2}}

Using Proposition \ref{prop:rec_termes}, we are now ready to state the next
important result, which will be key for the almost sure convergence of $(\tn)_{n \ge 1}$.
\begin{proposition}\label{prop:estimation_series} Under the assumptions of Proposition \ref{prop:rec_termes}, then:
	\begin{itemize}
		\item[$i)$] 
		Two constants $c(\theta_0,w_0)$ and $\kappa$  exist such that, for all $n\ge1$
	
		\begin{align*}
		\PE\left(\sum_{k = 1}^n \gamma_{k+1} \left[q_k \|\sqrt{w_k}\|^2+  \|\sqrt{|\nabla f(\theta_k)|}\|^4\right]\right) &\leq c(\theta_0,w_0)\exp\left(\kappa \sum_{k=1}^{n}  (s_k+t_k)  \right)
+\kappa \sum_{k=1}^{n} (s_k+t_k)
		\end{align*}
		
			\textcolor{black}{where
			$(t_n)_{n\ge 0}$ and $(s_n)_{n\ge 0}$ are the auxiliary sequences defined in Proposition \ref{prop:rec_termes}.}
		
%
		\item[$ii)$] If $\sum_{n \ge 1} (\gamma_{n+1}^2 + \gamma_{n+1} \sigma_{n+1}^2p_n) < + \infty$ , then almost surely:
		\begin{equation}\label{eq:series_ps}
		\sum_{n \ge 1} \gamma_{n+1} \left[q_n \|\sqrt{\wn}\|^2+  \|\sqrt{|\nabla f(\tn)|}\|^4\right]< + \infty \, a.s.
		\end{equation}
	\end{itemize}

\end{proposition}

\begin{proof}
	\underline{$\bullet$ Proof of $i)$.}
	Our proof relies on a Lyapunov function defined by:
	$$
	V_{a,b}(\theta,w):= \|\sqrt{w+\varepsilon}\|^2 + a f^2(\theta)+ b f(\theta)\|(w+\varepsilon)^{1/4}\|^2,$$
	with a careful tuning of $a$ and $b$. 
	
	Using $i)$, $iii)$ and $iv)$ of Proposition \ref{prop:rec_termes} and the fact that $f(\theta_{n})\le 1+f(\theta_{n})^2$, 
we deduce that a constant $\kappa$ that depends on $\cu,\cd,\ct$ and of the next choice of $a$ and $b$ exists such that:

\begin{align*}
\PE\left[V_{a,b}(\tnp,\wnp) \, \vert \Fn \right] &\leq V_{a,b}(\tn,\wn)\left[1+\kappa (\textcolor{black}{s_n} +t_n) \right] + \kappa (\textcolor{black}{s_n} +t_n) \\
&- q_n \gamma_{n+1} \|\sqrt{\wn}\|^2 \\
&+ \gamma_{n+1} \left[p_n \|\nabla f(\tn)\|^2-b \|\sqrt{|\nabla f(\tn)|}\|^4\right]\\
& + \gamma_{n+1} [b p_n - 2a ]  f(\tn)\left\|\dfrac{\nabla f (\theta_n)}{(w_n+\varepsilon)^{1/4}}\right\|^2.
\end{align*}
	We observe that for any vector $u$, we have $\|\sqrt{|u|}\|^4 \ge \|u\|^2$ and
	  that $(p_n)_{n \ge 1}$ is a bounded sequence, so that we can find $b$ large enough to have $b>2\sup_{n \ge 1}p_n$, and $2a \ge b p_n$ such that $\forall n \ge 1$: 
	\begin{align*}
	\PE\left[V_{a,b}(\tnp,\wnp) \, \vert \Fn \right] &\leq 
	V_{a,b}(\tn,\wn)\left[1+\kappa (t_n+s_n) \right] + \kappa (t_n+s_n)\\
	& - \gamma_{n+1}\left[ q_n \|\sqrt{\wn}\|^2 + \frac{b}{2}  \|\sqrt{|\nabla f(\tn)|}\|^4\right].
	\end{align*}
	We then conclude using a straightforward recursion. \hfill $\diamond$
	
	\underline{$\bullet$ Proof of $ii)$.}
	This point  proceeds with standard arguments: we use the Robbins Siegmund Lemma (see \cite{robbins1971convergence}): the series $\sum \gamma_{n+1}^2$ and $\sum\gamma_{n+1} \sigma_{n+1}^2p_n$ are convergent and obtain that:
	\begin{enumerate} 
		\item $V_{a,b}(\tn,\wn) \to V_{\infty}  $ a.s. (and in $L^1$) and $\sup_n \mathbb{E}[V_{a,b}(\tn,\wn)]<+\infty$ 
		\item More importantly, the next series are convergent:
		\begin{equation}\label{serie grad}
		\sum_{n\ge 0} \gamma_{n+1}\| \nabla f(\theta_n)\|^2 < \sum_{n\ge 0} \gamma_{n+1}\| \sqrt{|\nabla f(\theta_n)|}\|^4  <+\infty\ a.s. 
		\end{equation}
		and 
		\begin{equation}\label{serie w}
		\sum_{n\ge 0} \gamma_{n+1}q_n\| \sqrt{w_n}\|^2 <+\infty\ a.s.
		\end{equation}  
	\end{enumerate}	
	This ends the proof of $ii)$ \hfill $\diamond$
%
\end{proof}

We emphasize that $i),ii)$ are standard consequences of the Robbins-Siegmund approach on stochastic algorithms. In particular, the use of $i)$ with the calibrations of $(\gamma_n)_{1 \leq n \leq N},
(\sigma_n)_{1 \leq n \leq N}$ and $(p_n)_{1 \leq n \leq N}$ 
 that are given in the statement of Theorem \ref{thm: cv quanti} instantaneously lead to the conclusion of the proof.
We also point out that the basic fact to obtain this result is a tuning of the parameters thats leads to
$$
\sum_{n=1}^N \left( s_n \vee t_n\right) = \mathcal{O}_d(1).
$$ 
Considering constant step-size sequences, it entails the following constraints on $(p,\gamma,\sigma)$:
$$
 \gamma^2 \vee  \gamma p \sigma^2 d \vee  \gamma^2 \sigma^2 d = \mathcal{O}_d(N^{-1}),
$$
whereas the size of $N$ needed to obtain a $\delta$ approximation should verify that:
$$
\PE\left[\frac{1}{N} \sum_{k=1}^N \| \sqrt{|\nabla f(\theta_k)|}\|^4 \right] \leq \delta \Longleftrightarrow \frac{1}{N \gamma } \PE\left[\sum_{k=1}^N \gamma \| \sqrt{|\nabla f(\theta_k)|}\|^4\right] \leq \delta \Longleftarrow N \ge (\delta \gamma)^{-1}.
$$
The computational cost associated to a such $N$ is $N \sigma^{-2}$.
Defining $\gamma  = \tilde{\gamma} N^{-1/2}$, we then observe that $N$ should be chosen larger than $(\tilde{\gamma} \delta)^{-2}$, which entails a computational cost of $(\tilde{\gamma} \sigma \delta)^{-2}$. With this setting, the constraints on the sequences become:
$$
\tilde{\gamma}^2 \sigma ^2 \led \frac{1}{d} \quad \text{and} \quad \tilde{\gamma} \sigma^2 p \led \frac{1}{d \sqrt{N}}.
$$
Hence, the computational cost is lower bounded by $d \delta^{-2}$ and this bound may be achieved as soon as 
$$\tilde{\gamma}^2 \sigma^2 = \frac{1}{d} \quad \text{and} \quad p = \frac{\tilde{\gamma}}{\sqrt{N}},$$
which corresponds to the tuning described in $iii)$ and $iv)$ of Theorem \ref{thm: cv quanti}.


\subsection{Proof of Theorem \ref{thm: cv ps}}
\paragraph{Asymptotic pseudo-trajectory}
For the sake of convenience, from now on we denote $V_n = V_{a,b}(\tn,\wn)$. Since we are now interested in purely asymptotic result, we omit the dependency with $d$ in the bounds we obtain hereafter.
The main difficulty here is to convert the result of Proposition \ref{prop:estimation_series} $ii)$ into an a.s. convergence result on $(\tn)_{n \ge 1}$.
We remind the standard definition of asymptotic pseudo-trajectories of a semiflow $\Phi$.
\begin{definition}[Pseudo-trajectory]
	A continuous trajectory $Z$ is a pseudo-trajectory of $\Phi$ if for any finite time horizon $T>0$
	$$
	\lim_{t \longrightarrow + \infty} \sup_{0<u<T} |Z_{t+u}-\Phi_u(Z_t)| = 0.
	$$
\end{definition}
We refer to \cite{benaim1996asymptotic,benaim1999dynamics}  for further details.
We will use in particular the cornerstone result  which is reminded below:
\begin{theorem}[Theorem 3.2 of \cite{benaim1999dynamics}]\label{theo:pseudo_traj} 
	If $Z$ is an asymptotic pseudo-trajectory of $\Phi$ with a compact closure in $\rset^d$, then every limit point of $Z$ is a fixed point of $\Phi$.
\end{theorem}

In what follows we use the results of \cite{benaim1999dynamics} to show that a linear interpolation of the sequence $(Z_n)_{n\ge}$ is an asymptotic pseudo-trajectory of the flow induced by
the vector field $H:\mathbb{R}^{2d}\to \mathbb{R}^{2d}$ of our adaptive algorithm defined by:

\begin{equation}\label{def:H}
H(\theta,w)=\begin{pmatrix}
-\dfrac{\nabla f(\theta)}{\sqrt{w+\varepsilon}}\\p_{\infty}\oo{\nabla f(\theta)}-q_{\infty}w
\end{pmatrix}.
\end{equation}

Denote $\tau_0=0$, $\tau_n=\sum_{k=1}^{n} \gamma_k$ and consider  $(\bar{Z}_t)_{t\ge0}$ the continuous time process corresponding to a linear interpolation of $(Z_n)_{n\ge 0}$, given by :
$$\bar{Z}(\tau_n+s)= Z_n+s\dfrac{Z_{n+1}-Z_n}{\gamma_{n+1}}; \quad \forall n\in\mathbb{N};\ \forall \ 0\le s\le \gamma_{n+1}.$$

The Robbins-Siegmund Lemma ensures that the sequence $(V_n)_{n \ge 1}$ converges a.s. to a finite random variable $V_{\infty}$. Since all the terms of $V_n$ are positive, $(f(\theta_{n}))_{n\ge 0}$ and $(\|w_n\|)_{n\ge0}$ are a.s. bounded as well. The coercivity of $f$ implies thus that $(Z_n)_{n\ge 0}$ is a.s. bounded.  \\
The evolution of the sequence $(Z_n)_{n\ge0}$ can be written as:
$$Z_{n+1}=Z_n+\gamma_{n+1}(H(Z_n)+\dn),$$
where 
$\dn$ is a rest term:
\begin{equation}
\label{def:en}\dn=\begin{pmatrix}
\dfrac{\nabla f(\theta_n)-g_{n+1}}{\sqrt{w_n+\varepsilon}}\\p_n \oo{g}_{n+1}-p_{\infty}\oo{\nabla f(\theta_n)}-(q_n-q_{\infty})w_n
\end{pmatrix}.
\end{equation}
The next result allows us to control the rest term and ensures that the assumptions of Proposition 4.1 of \cite{benaim1999dynamics} are fulfilled, which in turn implies that $(\bar{Z}_t)_{t\ge0}$ is indeed an asymptotic pseudo trajectory of the flow induced by $H$.
\begin{lemma} Under the assumptions of Proposition \ref{prop:rec_termes}
	\label{lemme: Pseudo-trajectoire} and
 if $\sum\gamma_{k+1}\sigma_{k+1}^2 p_k < + \infty$ and $(p_n,q_n)\longrightarrow (p_\infty,q_\infty)$.
	Let $N(n,t)=\sup_{k\ge0} \{ t +\tau_n\ge \tau_k\}$ and $(\dn)_{n\ge1}$ defined in Equation \eqref{def:en}. For all $T>0$:
	\begin{equation}
	\limsup_{n\to \infty}\sup_{t\in[0,T]} \left\| \sum_{k=n+1}^{N(n,t)+1} \gamma_k e_k\right\|=0 \ a.s.
	\end{equation}
\end{lemma} 
\begin{proof}
	We consider a finite horizon $T>0$.
	In order to deal with the previous sum, we write $\dn$ as $a_n+\Delta M_{n+1}+c_n$, with  $a_n= \begin{pmatrix}
	0\\(p_n-p_{\infty})\oo{\nabla f(\theta_n)}-(q_n-q_{\infty})w_n
	\end{pmatrix}$, $\Delta M_{n+1}=\begin{pmatrix} \dfrac{\nabla f(\theta_n)-g_{n+1}}{\sqrt{w_n+\varepsilon}}\\ 0 \end{pmatrix}$  and $c_{n}=\begin{pmatrix} 0\\ p_n(\oo{\nabla f(\theta_n)}- \oo{g}_{n+1}) \end{pmatrix}.$	We use the fact that:  $$\left\| \sum_{k=n+1}^{N(n,t)+1} \gamma_k e_k\right\|\le  \left\| \sum_{k=n+1}^{N(n,t)+1} \gamma_k a_k\right\|+\left\| \sum_{k=n+1}^{N(n,t)+1} \gamma_k \Delta M_{k+1}\right\|+\left\| \sum_{k=n+1}^{N(n,t)+1} \gamma_k c_k\right\|,$$
	and proceed to upper bound each term on the right hand side. 
	
	\noindent $\bullet$  The convergence of the first term is mainly a consequence of Equations \eqref{serie grad}, \eqref{serie w} (convergence of the Robbins-Siegmund series), of the convergence of $(q_n)_{n \ge 1}$ towards and $q_{\infty}$ and the fact that $(p_n- p_{\infty})_{n \ge 1}$ is a bounded sequence :
	
	\begin{align*}
	\left\| \sum_{k=n+1}^{N(n,t)+1} \gamma_k a_k\right\|&\le  \left\| \sum_{k=n+1}^{N(n,t)+1} \gamma_k \left( (p_k-p_{\infty})\nabla f(\theta_k)^2-(q_k-q_{\infty})w_k \right)\right\|\\
	& \le \left\| \sum_{k=n+1}^{N(n,t)+1} \gamma_k (p_k-p_{\infty})\nabla f(\theta_k)^2\right\| + \left\| \sum_{k=n+1}^{N(n,t)+1} \gamma_k (q_k-q_{\infty}) w_k\right\|\\
	& \le \sum_{k=n+1}^{N(n,t)+1} \gamma_k |p_k-p_{\infty}| \left\|  \nabla f(\theta_k)^2 \right\|  +  \sum_{k=n+1}^{N(n,t)+1} \gamma_k |q_k-q_{\infty}|\left\| w_k\right\|
	\end{align*}
	
	Assumption $(\Hg-1)$ ensures that a constant $P$ exists such that $|p_n-p_{\infty}|<P$ and that for $n$ large enough $|q_n-q_{\infty}| \le q_n$.  Moreover, $\forall a,b\in\mathbb{R}_+$, $\sqrt{a+b}\le \sqrt{a}+\sqrt{b}$ thus $\| \oo{\nabla f(\theta_n)}\| \le \|\nabla f(\theta_n)\|^2$ and $\|w_n\| \le \| \sqrt{w_n}\|^2$. Inserting these bounds in the previous inequality gives:
	\begin{align*}
	\left\| \sum_{k=n+1}^{N(n,t)+1} \gamma_k a_k\right\|	
	&\le \sum_{k=n+1}^{N(n,t)+1} P \gamma_k  \left\|\nabla f(\theta_k)\right\|^2+ \sum_{k=n+1}^{N(n,t)+1} \gamma_k q_k \| \sqrt{w_k}\|^2.
	\end{align*}
	Using the convergence of the series \eqref{serie grad} and \eqref{serie w} we conclude that, $\forall t >0$\footnote{This last limit holds regardless $t<T$}:
	\begin{equation}
	\limsup_{n\to \infty}\left\| \sum_{k=n+1}^{N(n,t)+1} \gamma_k a_k\right\| =0.
	\end{equation} 
	\medskip
	
	\noindent $\bullet$ To control the second term we observe that $(\Delta M_{k+1})_{k\ge 0}$ is a sequence of martingale increments with a bounded second order moment since  
	$$ \forall k \ge 1 \qquad \mathbb{E}\left[ \frac{\xi_{k+1}}{\sqrt{w_k+\varepsilon}} \, | \mathcal{F}_k\right]=0,$$  and 
	\begin{align*}
	\sup_{k\ge0}\mathbb{E}\left[\left\| \frac{\xi_{k+1}}{\sqrt{w_k+\varepsilon}}\right\|^2\right]&\le \sup_{k\ge0} \frac{1}{\varepsilon}\mathbb{E}\left[\mathbb{E}\left[\left\| \xi_{k+1}\right\|^2 | \mathcal{F}_k\right]\right]\le \sup_{k\ge0} \frac{1}{\varepsilon} \sigma_{k+1}^2(d+ \PE[f(\theta_{k})]) <+\infty.
	\end{align*}

	The last inequality is implied by the fact that  $f(\theta_{k})\le 1+ f(\tn)^2\lesssim 1 +V_k$  and the Robbins Siegmund lemma guarantees that $\sup_{k\ge0} \PE[V_k]<+\infty$.
	Thus Corollary 11 of \cite{metivier1987theoremes} applies (see also Proposition 4.2 of \cite{benaim1999dynamics}) and we obtain that:  
	\begin{equation}
	\forall t>0 \qquad \limsup_{n\to \infty}\left\| \sum_{k=n+1}^{N(n,t)+1} \gamma_k \Delta M_{k+1}\right\| =0.
	\end{equation} 
	
	\noindent $\bullet$ The assumptions made on the noise sequence and the fact that $(\mathbb{E}[V_n])_{n \ge 1}$ is uniformly bounded ($V_n \longrightarrow V_{\infty}$ in $L^1$) implies that the last term can be handled using the same type of arguments.
	
	We start by decomposing $c_k$ as its expected value plus a martingale increment:
	\begin{align*}
	\left\| \sum_{k=n+1}^{N(n,t)+1} \gamma_k c_k\right\|
	&\le  \left\|\sum_{k=n+1}^{N(n,t)+1}\gamma_{k+1}p_k \mathbb{E}[ \oo{\xi_{k+1}}\, | \mathcal{F}_k] \right\| \\
	&+  \left\| \sum_{k=n+1}^{N(n,t)+1}\gamma_{k+1}p_k \left(2\xi_{k+1}\cdot \nabla f(\theta_k) + \oo{\xi_{k+1}}-\mathbb{E}[ \oo{\xi_{k+1}}| \mathcal{F}_k]\right)\right\|.
	\end{align*}
	Since $\sum_{n\ge 0}\gamma_{n+1}p_n\sigma_{n+1}^2<+\infty$ the first sum convergences to $0$ when $n$ goes to infinity. \\
	The terms of the second sum are martingale increments:  
	$$\mathbb{E}[2\xi_{k+1}\cdot\nabla f(\theta_k) + \oo{\xi_{k+1}}-\mathbb{E}[ \oo{\xi_{k+1}} \, \vert \mathcal{F}_k] \, \vert \mathcal{F}_k]=0.$$
	Using the fact that $(a+b)^2\le 2(a^2+b^2)$,  we get a first bound on their second order moments: 
	\begin{align*}
	\mathbb{E}[\|2\xi_{k+1}\cdot\nabla f(\theta_k) + \oo{\xi_{k+1}}-\mathbb{E}[ \oo{\xi_{k+1}}| \mathcal{F}_k] \|^2]
	&\le 8\left( \mathbb{E}[\|\xi_{k+1}\cdot \nabla f(\theta_k)\|^2]+ \mathbb{E}[\|\oo{\xi_{k+1}}\|^2] \right).
	\end{align*}
	Now using $\Hp$ for $p=2$ and Inequality $\eqref{eq:comparaison}$:
	\begin{align*}
	\mathbb{E}[\|\xi_{k+1}\cdot\nabla f(\theta_k)\|^2]&=\mathbb{E}\left[\sum_{i=1}^d \partial_i f(\theta_k)^2\mathbb{E}[\xi_{k+1,i}^2| \mathcal{F}_k] \right]\le \sigma_{k+1}^2\mathbb{E}[(d+f(\theta_{k}))\|\nabla f(\theta_k)\|^2]\\
	&\le c_f\sigma_{k+1}^2\mathbb{E}[f(\theta_k)(d+f(\theta_{k}))]\\
	&\lesssim c_f\sigma_{k+1}^2(1+ \mathbb{E}[V_k]).
	\end{align*}
	The second term can be dealt with in a similar manner, using the assumption $\Hp$ for $p=4$: 
	$$\mathbb{E}[\|\oo{\xi_{k+1}}\|^2]\le \sigma_{k+1}^4\mathbb{E}[\|\zeta_{k+1}\|^4]\le \sigma_{k+1}^4 \PE[(d+f(\theta_{k}))^2]\lesssim \sigma_{k+1}^4(1+ \PE[V_k]). $$
	The sequence $(\sigma_{n})_{n\ge1 }$ is bounded and according to the Robbins-Siegmund Theorem, this is also true for $(\mathbb{E}[V_n])_{n\ge1}$
	Consequently, we have:
	$$\sup_{k\ge0}\mathbb{E}\left[\left\|2\xi_{k+1}\cdot \nabla f(\theta_k) + \oo{\xi_{k+1}}-\mathbb{E}[ \oo{\xi_{k+1}}| \mathcal{F}_k] \right\|^2\right]<+\infty,$$
	and we can apply once more Corollary 11 of [MP] to conclude that:
	\begin{equation}
	\limsup_{n\to \infty}\left\| \sum_{k=n+1}^{N(n,t)+1} \gamma_k c_k\right\| =0, 
	\end{equation} 
	which ends  the proof.
\end{proof}

\paragraph{Proof of the almost sure convergence towards a critical point}

We now give the proof of the leading result of Section \ref{sec:as}.

\begin{proof}[Proof of Theorem \ref{thm: cv ps}] The proof is divided into two steps.
	
	\noindent
	\underline{$\bullet$ Identification of the possible limit points.} 
	The a.s. boundness of $(Z_n)_{n\ge 0}$ and Lemma \ref{lemme: Pseudo-trajectoire} show that the assumptions of Proposition 4.1 of \cite{benaim1999dynamics} hold, which implies that $(\bar{Z}_t)_{t\ge 0}$ is an asymptotic pseudo-trajectory of the differential flow induced by $H$ almost surely. It implies in particular that $(\wn)_{n \ge 1}$ is a.s. bounded.
	
	We shall deduct from Theorem \ref{theo:pseudo_traj} that all limit points $Z_{\infty}=(\theta_{\infty},w_{\infty})$ of $\bar{Z}_t$ are stationary points for the differential equation $\dot{z}=H(z)$ and thus that $H(Z_{\infty})=0$. Since $\|w_n\|$ is a.s. bounded, it implies that $\nabla f(\theta_{\infty})=0$.
	
	In the meantime, since $\nabla f(\theta_{n}) \to 0$ when $n\to \infty$, we observe that for any limit point $Z_{\infty}$,
	$H(Z_{\infty})=0$ also implies that $w_{\infty}=0$.
	
	\noindent
	\underline{$\bullet$ Convergence of $(\tn)_{n \ge 1}$.} 
	Thus (since $w_n$ is bounded) we have that $w_n$ converges to $0$. Hence:
	\begin{align*}
	\lim_{n\to\infty}V_{a,b}(\tn,\wn)&= \lim_{n\to\infty} \|\sqrt{\wn+\varepsilon}\|^2+af^2(\tn)+ b f(\tn)\|(\wn+\varepsilon)^{1/4}\|^2\\
	&=\lim_{n\to\infty} a f^2(\tn) + bd \sqrt{\varepsilon} f(\theta_n) +d\varepsilon =V_{\infty}.
	\end{align*}
	Since $a$ and $b$ are non-negative and $f(\tn)$ positive, 
	the last equality implies that the sequence $(f(\theta_{n}))_{n\ge 0}$ is a.s. convergent: 
	$$
	\lim_{n \longrightarrow + \infty} f(\tn) = f_{\infty} \quad a.s.
	$$
	
	Now, $(\theta_{n})_{n\ge 0}$ is an a.s. bounded sequence and the set of possible limit points for its sub-sequences is connected since:
	\begin{align*}
	\|\theta_{n+1}-\theta_{n}\| & = \left\|-\gamma_{n+1}\frac{g_{n+1}}{\sqrt{w_n+\varepsilon}} \right\|\\
	&\le  \frac{1}{\sqrt{\varepsilon}} \gamma_{n+1} \| g_{n+1}\|\\
	&\le \frac{1}{\sqrt{p_n\varepsilon}}\gamma_{n+1} \left\| \sqrt{ \frac{w_{n+1}-w_n}{\gamma_{n+1}}+q_n w_n} \right\|\\
	&\le \frac{1}{\sqrt{p_n\varepsilon}}\sqrt{\gamma_{n+1}} \| \sqrt{w_{n+1}-w_n+  \gamma_{n+1}q_n w_n}\| \\ 
	&\le  \frac{K}{\sqrt{\varepsilon}}\sqrt{\frac{\gamma_{n+1}}{p_n}} \longrightarrow 0 \quad \text{as} \quad n\longrightarrow + \infty.
	\end{align*}
	
	Since $ \{ \theta :  f(\theta)=f_{\infty}\}\cap\{ \theta : \nabla f (\theta)=0\}$ is locally finite, we can conclude that $(\tn)_{n \ge 1}$ has a unique adherence point and is a convergent sequence:
	$$
	\lim_{n \longrightarrow + \infty} \tn= \theta_{\infty} \qquad \text{with} \qquad \nabla f(\theta_{\infty})=0.
	$$
\end{proof}
 \section{Almost sure convergence to a local minimum\label{sec:instable}}
In this paragraph, we prove that the sequence $(\theta_n)_{n \ge 1}$ almost surely converges towards a local minimum of $f$ and cannot converge towards an unstable (hyperbolic) point of the dynamical system, \textit{e.g.} cannot converge towards a saddle point or a local maximum of $f$. 
%

\subsection{Unstable equilibria}

We begin with a simple statement that identifies the unstable points of the dynamical system 
\begin{equation}\label{def:dynamic}
(\dot{\theta}_t,\dot{w}_t)  = H(\theta_t,w_t) \quad \text{with} \quad
H(\theta,w)=\left(-\frac{\nabla f(\theta_t)}{\sqrt{w_t+\varepsilon}},p_{\infty} [\oo{\nabla f(\theta_t)}]-q_{\infty} w_t\right).
\end{equation}
These equilibria may correspond to repulsive equilibria or hyperbolic points (saddle points). The next proposition makes this last sentence more precise and introduce the cornerstone function $\eta$ that measures in each neighborhood of an unstable (or saddle) point the distance of any point $x$ to the local stable manifold in the expanding direction.
\begin{proposition}\label{prop:instable}
 Consider the dynamical system \eqref{def:dynamic},  and assume that $f$ is twice differentiable, then:

\begin{enumerate}[label=(\roman*)]
\item The equilibria are $(t,0)$ where $t$ is a critical point of $f$.
\item
If $t$ is a local maximum of $f$, the dynamical system is unstable near $(t,0)$.
\item
If $t$ is a local minimum of $f$, the dynamical system is stable near $(t,0)$.
\item \label{item: unstable} If $t$ is an  unstable equilibria, a compact neighborhood $\mathcal{N}$ of $(t,0)$ and a function $\eta$ exist such that:
\begin{enumerate}
\item $ \forall z \in \mathcal{N} \quad \forall u \in \rset^d \times\rset^d \quad \eta(z+u) \ge \eta(z) + \langle \nabla \eta(z),u\rangle - \k \|u\|^2$.
\item If $E_+$ is the eigenspace associated to the negative eigenvalues of  $D^2f(t)$, then:
$$
 \forall z \in \mathcal{N} \quad \forall u=(u_1,u_2) \in \rset^d \times\rset^d \quad 
\lfloor 
 \langle \nabla \eta(z),u\rangle\rfloor_+ \ge c_1 \|\pi_+(u_1)\|.$$
 \item \label{item: lyapunov_inverse} A constant $\kappa >0$ exists such that:
 \begin{equation}\label{eq:lyapunov_inverse}
 \forall z \in \mathcal{N} \qquad 
 \langle \nabla \eta(z),H(z)\rangle \ge \kappa \eta(z).
 \end{equation}
\end{enumerate}
\end{enumerate}
\end{proposition}

\begin{proof}
\underline{$\bullet$ Proof of $i)$.}
The proof of $i)$ is immediate by observing that $H(\theta,w)=0$ and $q_{\infty} \neq 0$ implies that $\nabla f(\theta)=0$ and $w=0$.

\noindent \underline{$\bullet$ Proof of $ii)$ and $iii)$.}
We use a linearization of the drift around an equilibria $(t,0)$. Since $t$ is a critical point of $f$, we observe that:
$$
\nabla f(t+h) = D^2 f(t) h+o(\|h\|).
$$
Consequently, we observe that
$(\nabla f(t+h))^2=(D^2 f(t)h)^2 = \mathcal{O}\left(\|h\|^2\right)$, which entails:
\begin{align*}
H(t+h,\omega)&= \left( - \frac{D^2 f(t) h}{\sqrt{\varepsilon}} + \mathcal{O}\left(\|h\| \|\omega\|\right), p_{\infty} \mathcal{O}(\|h\|^2) - q_{\infty} \omega \right)\\
& = \left( \begin{matrix}
- \frac{D^2 f(t)}{\sqrt{\varepsilon}} & 0 \\ 0 & - q_{\infty} I_d
\end{matrix}\right) \left( \begin{matrix}
h \\ \omega
\end{matrix} \right) + o\left(\|h\|+\|\omega\|\right).
\end{align*}
The conclusion follows from the spectral decomposition of $D^2 f(t)$.

\noindent \underline{$\bullet$ Proof of $iv)$.} The last point is a consequence of Proposition 3.1 of \cite{benaim1995dynamics} (see also Proposition 9.5 of \cite{benaim1999dynamics}). We only need to observe that the coordinates in  $u_2$  always correspond to the attractive manifold so that $\pi_+(u)= \pi_+(u_1)$.
\end{proof}
\begin{remark}
	This proposition holds for any $p_{\infty}\in \mathbb{R}_+$. In particular it does so for $p_{\infty}=0$. 
\end{remark}
\subsection{Preliminary estimates}

For a given integer $n_0$ when $(\theta_{n_0},w_{n_0})$ is in a neighborhood $\mathcal{N}$ (given by \ref{item: unstable} of Proposition \ref{prop:instable} ) of an unstable point, we introduce the exit time of $\mathcal{N}$ defined by:
\begin{equation}\label{def:T}
T_{n_0} := \inf \left\{ n \ge n_0 \, : \, (\theta_n,w_n) \notin \mathcal{N}\right\}.
\end{equation}
We shall observe that if
$\mathbb{P}\left(T_{n_0}<+ \infty \right)=1$, then $(\theta_n,w_n)$ cannot converge almost surely to the unstable point $(t,0)$ located in $\mathcal{N}$ since the exit time is almost surely finite.
For this purpose, we introduce the sequence of random variables $(X_n)_{n \ge n_0+1}$ defined by:
\begin{equation}\label{def:X}
\forall n \ge n_0 \qquad X_{n+1} :=[\eta(\tnp,\wnp)-\eta(\tn,\wn)] \mathbf{1}_{n < T_{n_0}} + \alpha_{n+1}\mathbf{1}_{n \ge  T_{n_0}}
\end{equation}
where $(\alpha_{n+1})_{n \ge 1}$ will be made explicit later on, and the associated cumulative sum:
\begin{equation}\label{def:S}
\forall n \ge n_0+1 \qquad S_n :=
\eta(\theta_{n_0},w_{n_0}) + \sum_{k=n_0+1}^n X_k
\end{equation}
We prove the following upper bound on the second order moment of $(X_n)_{n \ge n_0+1}$.

\begin{proposition}\label{prop:Xn2}
A constant $c>0$ exists such that:
\begin{equation*}\label{eq:esperance_Xn2}
\forall n \ge n_0 \qquad \PE[X_{n+1}^2 \, \vert \, \Fn] \leq c (\gamma_{n+1}^2 \vee \alpha_{n+1}^2).
\end{equation*}
\end{proposition}

\begin{proof}

We decompose $X_{n+1}$ according to the position of $n$ with respect to $T_{n_0}$:
$$X_{n+1} = X_{n+1} \mathbf{1}_{n < T_{n_0}} + X_{n+1} \mathbf{1}_{n \ge T_{n_0}} = 
 X_{n+1} \mathbf{1}_{n < T_{n_0}} + \alpha_{n+1} \mathbf{1}_{n \ge T_{n_0}} 
$$
If $n\ge T_{n_0}$, there is nothing to prove.

We then consider the case when $n < T_{n_0}$ and define  $m=\sup_{z \in \mathcal{N}}\|\nabla \eta(z)\|$. A first order Taylor expansion yields:
\begin{align*}
X_{n+1}^2  \mathbf{1}_{n < T_{n_0}} &=(\eta( \tnp,\wnp)-\eta(\tn,\wn))^2  \mathbf{1}_{n < T_{n_0}} \\
&\le m^2 [\|\tnp-\tn\|^2+\|\wnp-\wn\|^2] \mathbf{1}_{n < T_{n_0}} \\
&\le \gamma_{n+1}^2 m^2 \left(\left\|\frac{g_{n+1}}{\sqrt{w_n+\varepsilon}}\right\|^2+\|p_n\oo{g}_{n+1}-q_nw_n\|^2   \right)\mathbf{1}_{n < T_{n_0}} \\
&\le \gamma_{n+1}^2 \frac{m^2}{ \varepsilon} \left(\| g_{n+1}\|^2+2(p_n\|\oo{g}_{n+1}\|^2 + q_n\|w_n\|^2 ) \right) \mathbf{1}_{n < T_{n_0}}.
\end{align*}
When $n<T_{n_0}$, the process $Z_n=(\theta_n,w_n)\in \mathcal{N}$ so that $\|\wn\|^2$ is bounded. It remains to study the terms that involve $\|g_{n+1}\|^2$ and  $\|\oo{g}_{n+1}\|^2$. We shall observe that:
\begin{align*}
\mathbf{1}_{n < T_{n_0}} \PE[\|g_{n+1}\|^2\| \ \Fn]&= \mathbf{1}_{n < T_{n_0}} \left(\|\nabla f(\theta_n)\|^2+ \PE[\|\xi_{n+1}\|^2 | \ \Fn]\right)\\
&\le \mathbf{1}_{n < T_{n_0}} \left( \sup_{z \in \mathcal{N}} (\|\nabla f(\theta_n)\|^2+ \sigma_{n+1}^2 \right) \\
&\le K'  \mathbf{1}_{n < T_{n_0}}.
\end{align*}
because $(\sigma_{n})_{n\ge 1}$ is bounded by definition and $\nabla f$ is continuous and  $\mathcal{N}$ is compact.\\

We then study $\mathbf{1}_{n < T_{n_0}}\PE[\|\oo{g}_{n+1}\|^2\| \ \Fn]$ and observe that:
$$
\|\oo{g}_{n+1}\|^2 = \|\oo{(\nabla f(\tn)+\xi_n)}\|^2 \leq 4 [\|\oo{\nabla f(\theta_n)}\|^2+\|\oo{\xi_{n+1}}\|^2] = 4 [\|\oo{\nabla f(\theta_n)}\|^2+\sigma_{n+1}^4 \|\oo{\zeta_{n+1}}\|^2].
$$
Thus from the assumption $\mathbf{H}_{\sigma}^{\infty}$ it follows that:

$$\mathbf{1}_{n < T_{n_0}} \PE[\|X_{n+1}\|^2\| \ \Fn]\le  K \gamma_{n+1}^2\mathbf{1}_{n < T_{n_0}},$$ where $K$ is a constant that depends on $\nabla \eta$, $m$, $\varepsilon$, $\|p\|_{\infty}$ and $\|q\|_{\infty}$. This ends the proof.
\end{proof}

%
%

Below, we will use the consequence of $iv)-(a)$ of Proposition \ref{prop:instable}: if $z=(\theta,w)$ is a point in $\mathcal{N}$, then a constant $\k$ exists such that:
\begin{equation}\label{Ineg: alpha}
\forall (z,u)\in \mathcal{N}\times\rset^d: \quad \eta(z+u)-\eta(z)\ge \langle \nabla \eta(z),u\rangle - \k  \|u\|^{2}
\end{equation}

We write the joint evolution of the algorithm $(Z_n)_{n\ge1}$ as:
\begin{equation*}
Z_{n+1}=Z_n+ \gamma_{n+1} \left(H(Z_n)+ \Delta_{n+1}  + U_{n+1}\right),
\end{equation*}
where $(\Delta_{n+1})_{n \ge 1}$ is a sequence of martingale increment defined by:
$$
\Delta_{n+1} :=
\begin{pmatrix}
\Delta M_{n+1}\\ \Delta N_{n+1}
\end{pmatrix}
$$ with
\begin{itemize}
	\item $\Delta M_{n+1}=-\frac{\xi_{n+1}}{\sqrt{w_n+\varepsilon}}$
	\item $\Delta N_{n+1}= p_n(\oo{\xi_{n+1}}+2\nabla f(\theta_n)\cdot \xi_{n+1}-\mathbb{E}[\oo{\xi_{n+1}}|\mathcal{F}_n])$
\end{itemize}
and  $$U_{n+1}=
\begin{pmatrix}	
	0 \\
	(p_n-p_{\infty})\oo{\nabla f(\theta_n)}-(q_n-q_{\infty})w_n+p_n\mathbb{E}[\oo{\xi_{n+1}}|\Fn]
	\end{pmatrix}.$$ 

We introduce below the sequence $(\nu_n)_{n \ge 1}$ that stands for the convergence rate of  $(p_n)_{n \ge 1}$ towards  $p_{\infty}$ and $(q_n)_{n \ge 1}$ towards  $q_{\infty}$:
\begin{equation}\label{def:nu_n}
\nu_n := |p_n-p_{\infty}| \vee |q_n-q_{\infty}|.
\end{equation}
%

\begin{proposition}\label{prop:repulsion}
For a large enough non negative constant $c$ the sequence $(\delta_n)_{n \ge 1}$ defined by
$\delta_n \eqdef  c_{\delta} (\nu_n+p_n\sigma_{n+1}^2+\gamma_{n+1})$ is such that:
\begin{equation*}\label{eq:repulsion}
\mathbf{1}_{S_n\ge \delta_n} \mathbb{E}[X_{n+1} | \ \Fn]\ge 0,
\end{equation*}
\end{proposition}

\begin{proof}

We start once more from the decomposition of $(X_n)_{n \ge n_0+1}$ before or after $T_{n_0}$
 and observe that:
\begin{align}\label{eq:X_apres_T}
\PE[\1_{n > T_{n_0}} X_{n+1} \, \vert \Fn] = \alpha_{n+1} \1_{n > T_{n_0}} \ge 0.
\end{align}
The definition of the stopping time $T_{n_0}$ ensures that when $n < T_{n_0}$, $Z_n$ belongs to the neighborhood $ \mathcal{N}$ so we can apply Equation $\eqref{Ineg: alpha}$ and obtain the following lower bound:
\begin{align*}
 \PE[ \mathbf{1}_{n < T_{n_0}} X_{n+1}| \ \F_n]&=\mathbf{1}_{n < T_{n_0}}\PE[\eta(Z_{n+1})-\eta(Z_n)|\ \Fn]\\
  &\ge  \1_{n<T_{n_0}} \bigg( \PE\left[ \langle \nabla \eta(Z_n),\gamma_{n+1}(H(Z_n)+\begin{pmatrix}
 \Delta M_{n+1}\\\Delta N_{n+1}
  \end{pmatrix}
 +\begin{pmatrix}
0\\U_{n+1}
 \end{pmatrix}\rangle\ |\Fn \right]\\
  &-\k  \gamma_{n+1}^2\PE\left[\left\| H(Z_n)+\begin{pmatrix}
  \Delta M_{n+1}\\\Delta N_{n+1}
  \end{pmatrix}
  +\begin{pmatrix}
  0\\U_{n+1}
  \end{pmatrix}\right\|^2\ |\, \Fn\right] \bigg).
\end{align*}
We treat the two terms separately:
 \begin{align*}
 \1_{n<T_{n_0}} &\PE\left[ \langle \nabla \eta(Z_n),\gamma_{n+1}(H(Z_n)+\begin{pmatrix}
 \Delta M_{n+1}\\\Delta N_{n+1}
 \end{pmatrix}
 +\begin{pmatrix}
 0\\U_{n+1}
 \end{pmatrix}\rangle\ |\Fn \right]\\
 &=\1_{n<T_{n_0}}\gamma_{n+1} \bigg(\left\langle \nabla \eta(Z_n),H(Z_n)\right\rangle
 +\PE\left[ \langle \nabla \eta(Z_n),\begin{pmatrix}
 \Delta M_{n+1}\\\Delta N_{n+1}
 \end{pmatrix}\rangle| \F_n\right]
 + \langle \nabla \eta(Z_n),
 \begin{pmatrix}
 0\\U_{n+1}
 \end{pmatrix}\rangle \bigg)\\
  &\ge \1_{n<T_{n_0}}\gamma_{n+1} \left(k\eta(Z_n)
 + \langle \nabla \eta(Z_n),
 \begin{pmatrix}
 0\\U_{n+1}
 \end{pmatrix}\rangle \right),
 \end{align*}
 where the last inequality comes from 
  \ref{item: lyapunov_inverse} of Proposition \ref{prop:instable} (Equation \eqref{eq:lyapunov_inverse}). Using that $\|\nabla \eta\|$ is upper bounded on  $\mathcal{N}$ by $m$, the definition of  $T_{n_0}$ leads to:
\begin{equation}\label{eq:norme_nabla_eta}
\1_{n < T_{n_0}} \|\nabla \eta(Z_n)\| \le m.
\end{equation}
This inequality associated with the Cauchy-Schwarz inequality implies that:
\begin{align*}
\1_{n<T_{n_0}}\left|\langle \nabla \eta(Z_n),
\begin{pmatrix}
0\\U_{n+1}
\end{pmatrix}\rangle\right|&\le \1_{n<T_{n_0}} \| \nabla \eta(Z_n)\| \left| \left|\begin{pmatrix}
0\\U_{n+1}
\end{pmatrix}\right|\right|\\
&\le m \1_{n<T_{n_0}} \| U_{n+1}
\|\\
&\le m\1_{n<T_{n_0}} \| \nu_n \|\oo{\nabla f(\theta_n)}\|+\nu_n\|w_n\|+p_n\PE[\oo{\xi_{n+1}}| \F_n]\|\\
& \le m \1_{n<T_{n_0}} \left(\nu_n \sup_{z \in \mathcal{N}}\|\oo{\nabla f(\theta_n)}\|+\nu_n \sup_{z \in \mathcal{N}}\|w_n\|+p_n\PE[\|\oo{\xi_{n+1}}\|\  | \F_n]\right).
\end{align*}
Observing that $\|\oo{\xi_{n+1}}\|=\sqrt{\sum_{i=1}^d \xi_{n+1,i}^4}\le\sum_{i=1}^d \xi_{n+1,i}^2=\|\xi_{n+1}\|^2 $, we deduce that:
$$\PE[\|\oo{\xi_{n+1}}\|\  | \F_n]\le \sigma_{n+1}^2.$$ We then define  $k'=\sup_{(\theta,w) \in \mathcal{N}}\|\oo{\nabla f(\theta)}\| +\|w\|< + \infty$ (because  $\mathcal{N}$ is compact and $\nabla f$ is continuous). We deduce that:
\begin{equation*}
\1_{n<T_{n_0}}\left|\langle \nabla \eta(Z_n),
\begin{pmatrix}
0\\U_{n+1}
\end{pmatrix}\rangle\right| \le k' m \1_{n<T_{n_0}}\left(2 \nu_n +p_n\sigma_{n+1}^2\right).
\end{equation*}
Concerning the last term, it is sufficient to show that the expected value is bounded. We start by using the fact that $(a+b+c)^2 \leq 4 (a^2+b^2+c^2)$ to split the squared norm and to obtain that:
\begin{align*}
\PE[\| H(Z_n)+&\begin{pmatrix}
\Delta M_{n+1}\\\Delta N_{n+1}
\end{pmatrix}
+\begin{pmatrix}
0\\U_{n+1}
\end{pmatrix}\|^2\ | \Fn ]\\
&\le  4\left(\left\| H(Z_n)\right\|^2+ \PE\left[\left\|\begin{pmatrix}
\Delta M_{n+1}\\\Delta N_{n+1}
\end{pmatrix}\right\|^2 | \Fn\right]
+\left\|U_{n+1}\right\|^{2}\ \right).
\end{align*}
Since $H$ is continuous, $(\|H(Z_n)\|^2)_{n_0\le n<T_{n_0} }$ is a bounded sequence. Moreover, we have seen that when $n<T_{n_0}$:
 $$\|U_{n+1}\| \le k' (2\nu_n +\sigma_{n+1}^2p_n),$$
thus we only left to study $ \PE\left[\left\|\begin{pmatrix}
\Delta M_{n+1}\\\Delta N_{n+1}
\end{pmatrix}\right\|^2 | \, \Fn\right]$.
\begin{align*}
 \PE\left[\left\|\begin{pmatrix}
\Delta M_{n+1}\\\Delta N_{n+1}
\end{pmatrix}\right\|^2 | \Fn\right]& = \PE\left[\|
\Delta M_{n+1}\|^2+\|\Delta N_{n+1}\|^2 | \Fn\right]\\
&\le \frac{\sigma_{n+1}^2}{\varepsilon}+p_n^2\PE [\|\oo{\xi_{n+1}}+2\nabla f(\theta_n)\cdot\xi_{n+1}-\mathbb{E}[\oo{\xi_{n+1}}|\mathcal{F}_n]\|^2| \ \Fn]\\
&\le \frac{\sigma_{n+1}^2}{\varepsilon}+4p_n^2\left(  \PE [\|\oo{\xi_{n+1}}\|^2+2\|\nabla f(\theta_n)\cdot\xi_{n+1}\|^2| \ \Fn]+\mathbb{E}[\|\oo{\xi_{n+1}}\|^2|\mathcal{F}_n] \right)\\
&\le \frac{\sigma_{n+1}^2}{\varepsilon}+4p_n^2\left(  2\PE [\|\xi_{n+1}\|^4| \ \Fn]+2\PE [\|\nabla f(\theta_n)\cdot\xi_{n+1}\|^2| \ \Fn] \right).
\end{align*}
When $n< T_{n_0}$, we observe that:
  $$\PE [\|\nabla f(\theta_n)\cdot \xi_{n+1}\|^2| \ \Fn]\le \sup_{z \in \mathcal{N}} \|\nabla f (\theta)^2\|_{\infty}\PE [\|\xi_{n+1}\|^2| \ \Fn],$$ and $$\PE [\|\xi_{n+1}\|^4| \ \Fn]=\sigma_{n+1}^4\PE [\|\zeta_{n+1}\|^4| \ \Fn].$$ Again, using the compactness of  $\mathcal{N}$, the continuity of  $\nabla f$ and the assumption $\mathbf{H}_{\sigma}^\infty$ on the noise $(\zeta_n)_{n\ge1}$, we deduct that $K_2>0$ exists such that:
$$\PE\left[\|
\Delta M_{n+1}\|^2+\|\Delta N_{n+1}\|^2 | \Fn\right] \le K_2(\sigma_{n+1}^2+\sigma_{n+1}^4).$$
We now gather all the terms and use the fact that the sequences $(\sigma_{n})_{n \ge 0}$, $(p_n)_{n\ge0}$ and $(\nu_n)_{n\ge 0}$  are all bounded, to conclude that a constant  $K_3$ exists such that:
\begin{align}\label{eq:X_avant_T}
\PE[ \mathbf{1}_{n < T_{n_0}} X_{n+1}| \ \F_n]&\ge\mathbf{1}_{n < T_{n_0}}
\gamma_{n+1}\Bigg( k\eta(Z_n)-k'm (2\nu_n+p_n\sigma_{n+1}^2)-K_3 \gamma_{n+1} \Bigg).
\end{align}
Finally, as long as $n<T_{n_0}$, $S_n=\eta(Z_n)$ and thus setting $c_{\delta}> \max(2k'm,K_3)$ and $(\delta_n)_{n \ge 1}$ as defined in the statement ends the proof.
\end{proof}

We now study the evolution of $S_n^2$.
\begin{proposition}\label{prop:Diff_Sn2}
Assume that  
$\delta_n=  c (\nu_n+p_n\sigma_{n+1}^2+\gamma_{n+1})$ is such that $\gamma_{n+1}\delta_n^2=o(\gamma_{n+1}^2\sigma_{n+1}^2\wedge \alpha_n^2)$ 
 then:
\begin{equation*}\label{eq:diff_sn2}
\mathbb{E}[S_{n+1}^2-S_n^2 \, \vert \, \Fn] \gtrsim \{\gamma_{n+1}^2 \sigma_{n+1}^2\} \wedge \alpha_{n+1}^2.
\end{equation*}
\end{proposition}

\begin{proof}
Our starting point is:
\begin{align*}
S_{n+1}^2-S_n^2& = (S_n+X_{n+1})^2-S_n^2\\
& = X_{n+1}^2+2 S_n X_{n+1}\\
&= X_{n+1}^2 + 2 S_n (\1_{S_n \ge \delta_n} X_{n+1}+
\1_{S_n < \delta_n} X_{n+1}).
\end{align*}
We then use Proposition \ref{prop:repulsion} and get:
\begin{align}\label{ineg: Sn 1}
\PE[S_{n+1}^2-S_n^2 \, \vert \, \Fn] & \ge \PE[X_{n+1}^2  \, \vert \, \Fn] + 2 S_n \1_{S_n < \delta_n} \PE [X_{n+1} \, \vert \Fn].
\end{align}
To derive a lower bound of $2 S_n \1_{S_n < \delta_n} \PE [X_{n+1} \, \vert \Fn]$, we observe that $(S_n)_{n \ge 0}$ and $\eta$ are positive, so that if we use Equations \eqref{eq:X_avant_T} and \eqref{eq:X_apres_T}, we obtain:

\begin{align*}
2 S_n \1_{S_n < \delta_n} \PE [X_{n+1} \, \vert \Fn]& \ge - 2 S_n \1_{S_n < \delta_n} \PE [ \1_{n<T_{n_0}}|X_{n+1}| \, \vert \Fn] \\
& \ge - 2 \delta_n  \gamma_{n+1}[ k' m (2 \nu_n+p_n\sigma_{n+1}^2)+K_3 \gamma_{n+1}]\\
& \ge  - 2 \delta_n^2 \gamma_{n+1}.
\end{align*}

We are led to analyze $\PE[X_{n+1}^2]$.
According to the definition of $(X_n)_{n \ge n_0}$, to the definition of the hitting time  $T_{n_0}$ and to the construction of $\eta$, we observe that from Equation \eqref{Ineg: alpha}:
\begin{align*}
\1_{n < T_{n_0}} X_{n+1} & = \1_{n<T_{n_0}}[\eta( Z_{n+1})-\eta(Z_n)]\\
& = \1_{n<T_{n_0}}[\eta[ Z_{n}+(Z_{n+1}-Z_n)]-\eta(Z_n)]\\
& \ge \1_{n<T_{n_0}} \left[\langle \nabla \eta(Z_n),Z_{n+1}-Z_n\rangle - k \|Z_{n+1}-Z_n\|^2 \right]\\
& \ge \1_{n<T_{n_0}} \gamma_{n+1} \langle \nabla \eta(Z_n), H(Z_n) \rangle 
+ \1_{n<T_{n_0}} \gamma_{n+1} \langle \nabla \eta(Z_n),\Delta_{n+1}+U_n\rangle \\
&- k \1_{n<T_{n_0}} \gamma_{n+1}^2 \|H(Z_n)+\Delta_{n+1}+U_n\|^2 \\
& \ge  \kappa \gamma_{n+1} \eta(Z_n) \1_{n<T_{n_0}} + \1_{n<T_{n_0}} \gamma_{n+1} \langle \nabla \eta(Z_n),\Delta_{n+1}\rangle \\
& - \gamma_{n+1} \1_{n<T_{n_0}} \|U_n\| \|\nabla \eta(Z_n)\| 
- k \1_{n<T_{n_0}} \gamma_{n+1}^2 \|H(Z_n)+\Delta_{n+1}+U_n\|^2,
\end{align*}
where in the last line we used the reverting effect translated in Equation \eqref{eq:lyapunov_inverse} and the Cauchy-Schwarz inequality.
Since $\eta$ is positive, we then obtain that:
\begin{align}
\1_{n < T_{n_0}} X_{n+1} & \ge 
\1_{n<T_{n_0}}  \gamma_{n+1} \langle \nabla \eta(Z_n),\Delta_{n+1}\rangle \nonumber\\
& -\1_{n<T_{n_0}} \left[
 \gamma_{n+1} \|U_n\| \|\nabla \eta(Z_n)\| 
+ k   \gamma_{n+1}^2 \|H(Z_n)+\Delta_{n+1}+U_n\|^2\right].\label{eq:minoration_Xn_inter}
\end{align}

We denote the positive part of any real value $a$ by $\pos{a}$ and we use that $a \ge b \Longrightarrow \pos{a} \ge \pos{b}$ and the inequality
$$
\pos{a-|b|} \ge \pos{a} - |b|.
$$
Considering Equation \eqref{eq:minoration_Xn_inter}, we then observe that:
\begin{align*}
\pos{\1_{n < T_{n_0}} X_{n+1}} \ge& \1_{n < T_{n_0}} \gamma_{n+1} \left[ \pos{ \langle \nabla \eta(Z_n),\Delta_{n+1}\rangle} \right. \\
& \left.- 
\|U_n\| \|\nabla \eta(Z_n)\| 
- k \gamma_{n+1} \|H(Z_n)+\Delta_{n+1}+U_n\|^2\right]
\end{align*}
Once more the regularity of $\nabla \eta$, the compactness of $\mathcal{N}$ and the definition of $U_n$, guarantee that a constant $\kappa>0$ exists such that:
$$
\1_{n < T_{n_0}}\|U_n\| \|\nabla \eta(Z_n)\|  \leq \kappa (\nu_n +p_n\sigma_{n+1}^2).
$$
Computing the conditional expectation and using the arguments of \eqref{eq:X_avant_T}, we have
$$
\1_{n < T_{n_0}} \PE[ \|H(Z_n)+\Delta_{n+1}+U_n\|^2 \, \vert \Fn] < K_3,
$$
so that:
\begin{align*}
\PE[ \pos{\1_{n < T_{n_0}} X_{n+1}}\, \vert \Fn] \ge & 
\1_{n < T_{n_0}} \gamma_{n+1} \PE[ \pos{ \langle \nabla \eta(Z_n),\Delta_{n+1}\rangle} \, \vert \Fn] \\
& - \1_{n < T_{n_0}} \left[\kappa (\nu_n+p_n\sigma_{n+1}^2)\gamma_{n+1} + k K_3 \gamma_{n+1}^2\right].
\end{align*}
Using that when $n < T_{n_0}$, $Z_n \in \mathcal{N}$, we can apply $iv)-(b)$ of Proposition \ref{prop:instable} so that:
$$
\PE[ \pos{\1_{n < T_{n_0}} X_{n+1}}\, \vert \Fn] \ge 
\1_{n < T_{n_0}}\left( \gamma_{n+1} [\PE \|\pi_{+}(\Delta M_{n+1})\| \, \vert \Fn]
- \left[\kappa (\nu_n+p_n\sigma_{n+1}^2)\gamma_{n+1} + k K_3 \gamma_{n+1}^2\right]\right),
$$
where $\pi_+$ is the orthogonal projection on $E_+$, the eigenspace associated to the negative eigenvalues of $D^2 f(t)$.
For $n$ large enough, the almost sure convergence of $(w_n)_{n\ge 0}$ to $0$ and our elliptic assumption $\mathbf{H}_{\sigma}^\infty-ii)$ on the sequence $(\xi_{n+1})_{n \ge 0}$  yield:
$$
\PE[\|\pi_{+}(\Delta M_{n+1})\| \, \vert \Fn] \ge \frac{\sigma_{n+1}}{2},
$$
which entails:
\begin{equation}\label{eq:minoration_Xn}
\PE[ \pos{\1_{n < T_{n_0}} X_{n+1}}\, \vert \Fn] \ge 
\1_{n < T_{n_0}} \gamma_{n+1} \left[ \frac{\sigma_{n+1}}{2}- C (\nu_n+p_n\sigma_{n+1}^2+\gamma_{n+1}^2)\right].
\end{equation}
Decomposing now $X_{n+1}=\pos{X_{n+1}}-\pos{-X_{n+1}}$, we then deduce that:
\begin{align*}
\PE[ \1_{n < T_{n_0}} X_{n+1}^2\, \vert \Fn]& = \PE[ \1_{n < T_{n_0}} \pos{X_{n+1}}^2\, \vert \Fn] + \PE[ \1_{n < T_{n_0}} \pos{-X_{n+1}}^2\, \vert \Fn] \\
& \ge \PE[ \1_{n < T_{n_0}} \pos{X_{n+1}}^2\, \vert \Fn] \\
& \ge \PE[ \1_{n < T_{n_0}} \pos{X_{n+1}}\, \vert \Fn]^2 \\
& \ge \1_{n < T_{n_0}} \gamma_{n+1}^2 \left[ \frac{\sigma_{n+1}}{2}- C (\nu_n+p_n\sigma_{n+1}^2+\gamma_{n+1}^2)\right]^2.\\
&\gtrsim  \1_{n < T_{n_0}} \gamma_{n+1}^2 \sigma_{n+1}^2.
\end{align*}
The last line is justified by the assumption $\gamma_{n+1}\delta_n^2 =o(\gamma_{n+1}^2\sigma_{n+1}^2)$ which ensures that $\delta_n=o(\sigma_n)$.
We shall also observe that:
$$
\PE[ \1_{n \ge T_{n_0}} X_{n+1}^2\, \vert \Fn] = \alpha_{n+1}^2.
$$
We then deduce that:
$$
\PE[X_{n+1}^2 \, \vert \Fn] \gtrsim \{\gamma_{n+1}^2 \sigma_{n+1}^2\} \wedge \alpha_{n+1}^2.
$$
Since
$\gamma_{n+1}\delta_n^2 =o(\gamma_{n+1}^2\sigma_{n+1}^2\wedge \alpha_{n}^2)$,
we can now conclude by inserting the previous bounds in the inequality $\eqref{ineg: Sn 1}$.

\end{proof}


\subsection{End of the proof}

We mimic the strategy of Lemma 9.6 of \cite{benaim1999dynamics} and of Theorem 3.2 of \cite{GPS2018} with the help of the two sequences defined in \eqref{def:S} and \eqref{def:X}. 
$$
S_n = \eta(Z_{n_0}) + \sum_{k=n_0+1}^n X_k \quad \text{with} \quad X_{n+1} = (\eta(Z_{n+1})-\eta(Z_n)) \1_{n < T_{n_0}} + \alpha_{n+1} \1_{n \ge T_{n_0}}.
$$
We summarize the preliminary results proven in the previous subsection as they will we used in what follows: 
\begin{itemize}
\item[$(I_1)$]
Proposition \ref{prop:Xn2} states that :
$$\PE[X_{n+1}^2 \, \Fn] \leq c (\gamma_{n+1}^2 \vee \alpha_{n+1}^2).$$
\item[$(I_2)$] 
Proposition \ref{prop:repulsion} yields that if $\delta_n = c_{\delta}(\nu_n+p_n\sigma_{n+1}^2+\gamma_{n+1}^2)$, then $$\1_{S_n \ge \delta_n} \PE[X_{n+1}\,\vert\Fn] \ge 0.$$
\item[$(I_3)$]
Proposition \ref{prop:Diff_Sn2} yields:  if $\gamma_{n+1}\delta_n^2=o(\gamma_{n+1}^2\sigma_{n+1}^2\wedge \alpha_n^2)$ (which also means that $\delta_n = o(\sigma_n)$), then $$\PE[S_{n+1}^2-S_n^2 \, \vert \Fn] \gtrsim \{\gamma_{n+1}^2 \sigma_{n+1}^2\} \wedge \alpha_{n+1}^2.$$
\end{itemize}

For any integer $q$, we consider  an integer $n \ge q$ and introduce the two sequences $(u_n)_{n \ge q}$ and $(\bar{U}_n)_{n \ge q}$ defined by:
$$
u_n = \sum_{i \ge n} \left(\{\gamma_{i+1}^2 \sigma_{i+1}^2\} \wedge \alpha_{i+1}^2\right) \qquad \text{and} \qquad \bar{U}_n = \sum_{i=0}^n \left(\{\gamma_{i+1}^2 \sigma_{i+1}^2\} \wedge \alpha_{i+1}^2\right).
$$
With the framework of Theorem \ref{thm:conv min} in mind 
we suppose that
$$\gamma_n= \gamma_1 n^{-\beta}, \quad \sigma_n = \sigma_1 n^{-s} \quad \mbox{and } \quad p_n=p_1n^{-r} $$
with  $\beta \in (1/2,1)$, $s\ge0$ and $r>0$.
Let us point out other restrictions on the choices of these parameters in light of previous assumptions: 

A first restriction follows directly from $(\Hg-3)$, namely that  $$2s+r+\beta>1$$ (to ensure that $\sum \gamma_{n+1}p_n \sigma_{n+1}^2< +\infty$).
Furthermore, Theorem \ref{thm: cv ps} demands that $\gamma_{n}/p_n\to 0$, implying that $$r<\beta.$$ 
To ease notations it is convenient to assume that the sequence $(q_n)$ converges to its limit $q_{\infty}>0$ at least as fast as $(p_n)$ goes to $0$ : $\nu_n=\mathcal{O}(p_n).$ With this in mind, the condition imposed by Proposition \ref{prop:Diff_Sn2}, $\delta_n=o(\sqrt{\gamma_{n+1}}\sigma_{n+1})$ translates to $r+s>\beta/2$,  $r>\beta/2+s$ and  $2\beta>\beta/2+s$, ie $\beta>2/3s$. As soon $\beta >s$ all these conditions come down to: 
$$ r>\beta/2+s$$
Other restrictions will be added in the following propositions and are  summarized in Theorem \ref{thm:conv min}.
%

\medskip

%


For any positive real value $b>0$ and any integer $q>0$, we consider the sequence of stopping times $(T_b^q)_{q \ge 0}$ defined by:
$$
T_b^q := \inf \left\{ i \ge q \, : S_i \ge \sqrt{ b u_i} \right\}.
$$
The stopping times $T_b^q$ stands for the first time the sequence $(S_n)_{n \ge 1}$ becomes larger than the threshold $(\sqrt{u_n})_{n \ge 1}$, which converges towards zero with a controlled rate.
We prove the following result.
\begin{proposition}\label{prop:T_fini}
Assume that $\sigma_i^2 \sim \sigma_1^2 i^{-2s}$ with $0\le s<1/2$ and $\gamma_i= \gamma_1 i^{-\beta}$, $r>\beta/2+s$ and choose $(\alpha_{n})_{n \ge 1}$ as $\forall n \ge 0 \,:\alpha_{n+1} = \gamma_{n+1} \sigma_{n+1}$.
Then a  small enough $b>0$ exists and a  large enough $q$ such that:
$$
\mathbb{P}(T_b^q < +\infty \, \vert \mathcal{F}_q) \ge \frac{1}{2}.
$$
\end{proposition}

\begin{proof}
Our starting point is the lower bound given by Proposition \ref{prop:Diff_Sn2} and
 we observe that a small enough $a>0$ exists such that:
$$
\mathcal{M}_n:=S_n^2- a \bar{U}_n
$$
is a submartingale since:
\begin{align*}
\PE[\mathcal{M}_{n+1}\, \vert \Fn] &= \PE[S_{n+1}^2-a \bar{U}_{n+1} \, \vert \Fn]\\
& \ge S_{n}^2 + a  \left(\{\gamma_{n+1}^2 \sigma_{n+1}^2\} \wedge \alpha_{n+1}^2\right) - a (\bar{U}_{n} +\{\gamma_{n+1}^2 \sigma_{n+1}^2\} \wedge \alpha_{n+1}^2) \\
& = \mathcal{M}_n.
\end{align*}
We consider an integer $n \ge q+1$ and apply the Optional Stopping Theorem and verify that:
\begin{align}
\PE[\mathcal{M}_{n \wedge T_b^q}] - \mathcal{M}_{q} \, \vert \mathcal{F}_q] \ge 0 & \Longleftrightarrow \PE[S_{n \wedge T_b^q}^2 - S_{q}^2 \, \vert \mathcal{F}_q] \ge a \PE\left[\sum_{i=q}^{n \wedge T_b^q} \{\gamma_{i+1}^2 \sigma_{i+1}^2\} \wedge \alpha_{i+1}^2 \, \vert \mathcal{F}_q\right] \nonumber\\
&
\Longleftrightarrow \PE[S_{n \wedge T_b^q}^2 - S_{q}^2 \, \vert \mathcal{F}_q] \ge a \mathbb{P}\left( T_b^q > n \, \vert \mathcal{F}_q\right) \sum_{i=q}^{n } \{\gamma_{i+1}^2 \sigma_{i+1}^2\} \wedge \alpha_{i+1}^2. \label{eq:optional_stoping}
\end{align}
In the meantime, we observe that:
\begin{align*}
\{S_{n \wedge T_b^q}\}^2 - \{S_q\}^2& = \{S_{n \wedge T_b^q}\}^2  - \{S_{(n \wedge T_b^q)-1}\}^2 + \{S_{(n \wedge T_b^q)-1}\}^2 - \{S_q\}^2 \\
& = \{S_{(n \wedge T_b^q)-1}+X_{n \wedge T_b^q}\}^2 - \{S_{(n \wedge T_b^q)-1}\}^2 + \{S_{(n \wedge T_b^q)-1}\}^2 - \{S_q\}^2\\
& \le  \{S_{(n \wedge T_b^q)-1}\}^2 +2 S_{(n \wedge T_b^q)-1} X_{n \wedge T_b^q} + \{X_{n \wedge T_b^q}\}^2\\
&\le 2\left(\{S_{(n \wedge T_b^q)-1}\}^2+\{X_{n \wedge T_b^q}\}^2\right)\\
&\le 2bu_{(n \wedge T_b^q)-1}+2 \{X_{n \wedge T_b^q}\}^2,
\end{align*}
where the last inequality is a consequence of the definition of the stopping time $T_b^q$. Since  $(u_n)_{n \ge 0}$ is a decreasing sequence, we then have that:
\begin{equation}\label{ineg:dec S2}
\{S_{n \wedge T_b^q}\}^2 - \{S_q\}^2
\le 2bu_{q-1}+2 \{X_{n \wedge T_b^q}\}^2,
\end{equation}
and we are led to upper bound the term $\{X_{n \wedge T_b^q}\}^2$. 

%

 The definition of $(X_n)_{n \ge 1}$ yields:
\begin{align*}
X_{n \wedge T_b^q}^2=\left(\eta(Z_{n \wedge T_b^q})-\eta(Z_{(n \wedge T_b^q)-1})\right)^2\1_{({n \wedge T_b^q})-1<T_{n_0}}+\alpha_{n \wedge T_b^q}^2\1_{(n \wedge T_b^q)-1\ge T_{n_0}}.
\end{align*}
Using a similar argument as the one used  in the proof of Proposition \ref{prop:Xn2}, a Taylor expansion associated with the smoothness of $\eta$ and $f$ leads to:
%

\begin{align*}
\1_{(n \wedge T_b^q)-1<T_{n_0}}\left(\eta(Z_{n \wedge T_b^q})-\eta(Z_{{n \wedge T_b^q}-1})\right)^2&\lesssim \1_{{n \wedge T_b^q}-1<T_{n_0}} \gamma_{n \wedge T_b^q}^2\left(1+\|\xi_{n \wedge T_b^q}\|^2+\|\xi_{n \wedge T_b^q}^2\|^2\right)\\
&\lesssim \1_{{n \wedge T_b^q}-1<T_{n_0}} \gamma_{n \wedge T_b^q}^2\left(1+\sigma_{n \wedge T_b^q}^2\|\zeta_{n \wedge T_b^q}\|^2+\sigma_{n \wedge T_b^q}^4\|\zeta_{n \wedge T_b^q}\|^4\right) \\
& \lesssim \gamma_{q+1}^2+\sum_{ i \ge q}^n \gamma_{i+1}^2 \sigma_{i+1}^2 \|\zeta_{i+1}\|^2 + \sum_{ i \ge q}^n \gamma_{i+1}^2 \sigma_{i+1}^4 \|\zeta_{i+1}\|^4,
\end{align*} 
where in the last line we used that $(\gamma_{n})_{n \ge q+1}$ is a decreasing sequence and a rough upper bound of $\sigma_{n \wedge T_b^q} \|\zeta_{n \wedge T_b^q}\|$.
We then compute the expectation with respect to $\mathcal{F}_q$ and observe that   $\mathbb{E}[ \| \zeta_{i+1} \|^2\ | \mathcal{F}_q]\le 1$ and $\mathbb{E}[ \| \zeta_{i+1} \|^4\ | \mathcal{F}_q]\le C$ and thus :
$$\mathbb{E}[ X^2_{n\wedge T_b^q}\ | \mathcal{F}_q] \lesssim \alpha_{q}^2 + \gamma_{q+1}^2+u_{q}.$$
This last bound together with inequality \eqref{ineg:dec S2} implies that a constant $C>0$ exists such that:
\begin{align*}
\mathbb{E}[\{S_{n \wedge T_b^q}\}^2 - \{S_q\}^2| \F_q] \le C\left( bu_{q}+ 
\alpha_{q}^2+ \gamma_{q+1}^2\right).
\end{align*}
Finally we obtain from \eqref{eq:optional_stoping} that:
\begin{align*}
a \mathbb{P}\left( T_b^q > n \, \vert \mathcal{F}_q\right) \sum_{i=q}^{n } \{\gamma_{i+1}^2 \sigma_{i+1}^2\} \wedge \alpha_{i+1}^2\le C\left( bu_{q}+ 
\alpha_{q}^2+ \gamma_{q+1}^2\right).
\end{align*}
We therefore deduce an upper bound on the probability that the stopping time is larger than $n$:
\begin{align*}
 \mathbb{P}\left( T_b^q > n \, \vert \mathcal{F}_q\right) &\le \frac{ C\left( bu_{q}+ 
 	\alpha_{q}^2+ \gamma_{q+1}^2\right)}{a\sum_{i=q}^{n } \{\gamma_{i+1}^2 \sigma_{i+1}^2\} \wedge \alpha_{i+1}^2} = \frac{Cb}{a} \frac{u_{q}}{u_q-u_n} + \frac{\alpha_{q}^2+ \gamma_q^2 }{a (u_q-u_n)} 
\end{align*}
We then take the limit $n \longrightarrow  +\infty$ in the previous inequality and obtain that:
\begin{align*}
\mathbb{P}\left( T_b^q =+\infty\, \vert \mathcal{F}_q\right) &\le \frac{ Cb}{a}+  C \frac{\alpha_{q}^2+ \gamma_q^2  }{a u_q},
\end{align*}
because $\lim_{n \longrightarrow + \infty} u_n = 0$.
Choosing  now $\alpha_{n+1}=\gamma_{n+1}\sigma_{n+1}$, we then observe that $u_q= \sum_{i \ge q} \gamma_{i+1}^2 \sigma_{i+1}^2$ and then the second term on the right hand side goes to zero as soon as: 
$$
\gamma_q^2 = o\left( \sum_{i \ge q} \gamma_{i+1}^2 \sigma_{i+1}^2\right).
$$
Since we have chosen  $\gamma_q = \gamma_1 q^{-\beta}$ and $\sigma_q= \sigma_1 q^{-s}$, we verify that a constant $\upsilon$ exists such that:
$$u_q \sim \upsilon q^{1-2(\beta+s)}.$$
Therefore, the condition $\gamma_{q}^2 =o(u_q)$ boils down to  $-2\beta<1-2(\beta+s)  \Leftrightarrow s<1/2$. 

Hence, when $s<1/2$ we can conclude the proof of the proposition by setting $b$ small enough ($b < a/3C$ for example).
\end{proof}

The next result states that $S_n$ may remain larger than $\frac{1}{2}\sqrt{b u_q}$ with a positive probability when $S_q \ge \sqrt{b u_q}$. For this purpose, we introduce 
$$
\mathcal{S} := \inf \left\{ n > q \, : \, S_n < \frac{1}{2} \sqrt{b u_q} \right\}
$$
that stands for the first time $(S_n)_{n \ge q}$ comes back below the threshold $\sqrt{b u_q}$.
\begin{proposition}\label{prop:S_infini}
Assume that $\alpha_{n+1} = \sigma_{n+1} \gamma_{n+1}$ and 
 $\beta+s-r < \frac{1}{2}$, $\beta \le  1-2s$  and that $\nu_n = o(\sqrt{\gamma_{n}} \sigma_n)$. Then there exits $q$ large enough and a constant $c>0$  such that 
$$
\1_{S_q \ge  \sqrt{b u_q}} \mathbb{P}\left(\mathcal{S} = + \infty \, \vert \mathcal{F}_q \right) \ge \1_{S_q \ge  \sqrt{b u_q}} \frac{b}{b+c}
$$
\end{proposition}

\begin{proof}
Since $\delta_n \sim p_n\sigma_n^2 + \nu_n+\gamma_{n}$, we conclude with our settings that:
$
\delta_n \sim n^{-r} + n^{-\beta}.
$
Moreover, the choice $\alpha_{i+1}=\gamma_{i+1}\sigma_{i+1}$ implies that:
$$u_q=\sum_{i \ge q+1}\gamma_{i+1}^2\sigma_{i+1}^2\sim n^{1-2(\beta+s)}.$$
We then observe  that $\delta_n=o(\sqrt{u_q})$ as soon as:
$$
n^{-r}+n^{-\beta} = o(n^{(1-2(\beta+s))/2}) \Longleftrightarrow \beta+s-r < \frac{1}{2},
$$
since $\beta > \frac{2(\beta+s)-1}{2}$ always holds when $s < \frac{1}{2}$.

From our previous remark, we observe that for $q$ large enough if $S_q$ is greater than $\frac{1}{2}\sqrt{ b u_q}$ then $S_q \ge \delta_{q}$, so Proposition \ref{prop:repulsion} implies that
$(S_{n \wedge \mathcal{S}})_{n \ge q}$ is a submartingale. For such a choice of $n$ and $q$,  we can use the Doob decomposition and write that:
$$
S_{n \wedge \mathcal{S}} = W_n +I_n,
$$
where $I_n$ is an increasing $\Fn$ predictable process with $I_q=0$ and $W_n$ is a martingale.
We then observe that
\begin{equation*}
\mathbb{P}\left(\mathcal{S} = + \infty \, \vert \mathcal{F}_q \right)  = 
\mathbb{P}\left( \forall n \ge q: \, S_n \ge \frac{1}{2} \sqrt{b u_q}\right)
\ge \mathbb{P}\left( \forall n \ge q: \, W_n \ge \frac{1}{2} \sqrt{b u_q}\right)
\end{equation*}
Furthermore, if $S_q \ge \sqrt{b u_q}$, then $W_q \ge  \sqrt{b u_q}$, which entails:
\begin{equation}\label{ineg: Sinf}
\1_{S_q \ge  \sqrt{b u_q}} \mathbb{P}\left(\mathcal{S} = + \infty \, \vert \mathcal{F}_q \right) \ge \1_{S_q \ge  \sqrt{b u_q}} 
\mathbb{P}\left(\forall n \ge q: \, W_n-W_q \ge -\frac{1}{2} \sqrt{b u_q}\right).
\end{equation}
Using the fact that $(W_n)_{n\ge q}$ is a martingale and the definition of $S_n$, one can verify that:  
$$
\PE\left[ (W_n-W_q)^2 \, \vert \mathcal{F}_q \right] \lesssim \sum_{i=q}^{n-1}\mathbb{E}[X^2_{i+1}| \mathcal{F}_q].
$$
The upper bound obtained in Proposition \ref{prop:Xn2} is not sharp enough to be directly applied here, so in order to deal with the term on the right hand side we return to the upper bound obtained in the proof of Proposition \ref{prop:Xn2} which gives for all $i$:
\begin{equation}\label{boundXi2}
X_{i+1}^2\1_{i<T_{n_0}}\lesssim \gamma_{i+1}^2 \left(\|g_{i+1}\|^2+2(p_i\|\oo{g_{i+1}}\|^2+q_i\|w_i\|)\right)\1_{i<T_{n_0}}.
\end{equation}
%
Now, for $i\ge q$ we have :
\begin{align*}
\mathbb{E}[\|g_{i+1}\|^2 | \mathcal{F}_q]&=\mathbb{E}[\mathbb{E}[\|g_{i+1}\|^2| \mathcal{F}_i]| \mathcal{F}_q]\\
&\le 2(\|\nabla f(\theta_i)\|^2+\mathbb{E}[\mathbb{E}[\|\xi_{i+1}\|^2| \mathcal{F}_i]| \mathcal{F}_q])\\
&\le  2(\|\nabla f(\theta_i)\|^2+\sigma_{i+1}^2),
\end{align*}
and 
\begin{align*}
\mathbb{E}[\|\oo{g_{i+1}}\|^2 | \mathcal{F}_q]&\le 8(\|\oo{\nabla f(\theta_i)}\|^2+\sigma_{i+1}^4\mathbb{E}[\mathbb{E}[\|\zeta_{i+1}^2\|^2| \mathcal{F}_i]| \mathcal{F}_q])\\
&\lesssim \|\oo{\nabla f(\theta_i)}\|^2+\sigma_{i+1}^4.
\end{align*}
The sequence $(\sigma_i)_{i\ge 0}$ is bounded, $(p_i)_{i\ge 0}$ converges to $0$ and Theorem \ref{thm: cv ps} shows that $(\|\nabla f(\theta_i)\|)_{i\ge 0}$ converge almost surely to $0$, so there exists $q>0$ such that $\forall i\ge q$: 
$$p_i\sigma_{i+1}^4 \le \sigma_{i+1}^2 \quad \mbox{and}\quad  \|\oo{\nabla f(\theta_i)}\|^2 \le \|\nabla f(\theta_i)\|^2.$$
Inserting this into our previous bound \eqref{boundXi2} we get that (for $q$ large enough and $i\ge q$)
\begin{equation*}
X_{i+1}^2\1_{i<T_{n_0}} \lesssim \gamma_{i+1}^2 \left(\|\nabla f(\theta_i)\|^2+q_i\|w_i\|+\sigma_{i+1}^2 \right)\1_{i<T_{n_0}}.
\end{equation*}
Since $\alpha_{i+1}=\gamma_{i+1}\sigma_{i+1}$ we have that:
\begin{equation*}
X_{i+1}^2 \lesssim \gamma_{i+1}^2 \left(\|\nabla f(\theta_i)\|^2+q_i\|w_i\|+\sigma_{i+1}^2 \right), 
\end{equation*}
which implies that:
\begin{align*}
\PE\left[ (W_n-W_q)^2 \, \vert \mathcal{F}_q \right] &\lesssim \sum_{i=q}^{n-1}  \gamma_{i+1}^2 \sigma_{i+1}^2 +\sum_{i=q}^{n-1}  \gamma_{i+1}^2 \left(\|\nabla f(\theta_i)\|^2+q_i\|\sqrt{w_i}\|^2 \right).
\end{align*}
The first series of the right-hand side is handled  with  the simple bound $\sum_{i=q}^{n-1}  \gamma_{i+1}^2 \sigma_{i+1}^2 \lesssim q^{1-2(\beta+s)}$.
%
%
%
In dealing with the second term of the right-hand side we use the almost sure convergence of the series stated in \eqref{eq:series_ps} of Proposition \ref{prop:estimation_series} and the fact that $(\gamma_{n})_{n\ge 0}$ is a decreasing sequence:
\begin{align*}
\sum_{i=q}^{n-1}  \gamma_{i+1}^2 \left(\|\nabla f(\theta_i)\|^2+q_i\|\sqrt{w_i}\|^2 \right)&\le \gamma_{q+1} \sum_{i=q}^{n-1}  \gamma_{i+1} \left(\|\nabla f(\theta_i)\|^2+q_i\|\sqrt{w_i}\|^2 \right)\\
&\lesssim \gamma_{q+1},
\end{align*}

These two ingredients yield:
\begin{equation*}
\PE\left[ (W_n-W_q)^2 \, \vert \mathcal{F}_q \right] \lesssim \sum_{i=q}^{n-1}  \gamma_{i+1}^2 \sigma_{i+1}^2 + \gamma_{q+1}.
\end{equation*} 
Thus there exists a constant $a>0$ such that for all $n\ge q$:
\begin{equation*}
\PE\left[ (W_n-W_q)^2 \, \vert \mathcal{F}_q \right] \le a (u_q+\gamma_q).
\end{equation*} 
Using a similar argument as the one of \cite{benaim1995dynamics} or \cite{GPS2018}, for all $h,t \in \mathbb{R}_+$ we deduce that:
\begin{align*}
\mathbb{P}\left( \inf _{q \le i \le n} (W_i-W_q)\leq - h \ | \mathcal{F}_q \right)&\le \mathbb{P}\left( \sup_{q \le i \le n} |W_i-W_q-t|\ge h+t \ | \mathcal{F}_q \right)\\
&\le 
\dfrac{\mathbb{E}[(W_n-W_q)^2 \, \vert \mathcal{F}_q]+t^2}{(h+t)^2}.
\end{align*}
Setting $t=a(u_q+\gamma_q)/h$ in the last term we have :

\begin{align*}
\dfrac{\mathbb{E}[(W_n-W_q)^2 \, \vert \mathcal{F}_q]+t^2}{(h+t)^2}
&\le 
\dfrac{a(\gamma_{q}+u_q)+ \frac{a^2(\gamma_{q}+u_q)^2}{h^2}}{(h+\frac{a(\gamma_{q}+u_q)}{h})^2}\\
&= \frac{a(\gamma_{q}+u_q)h^2+a^2(\gamma_{q}+u_q)^2}{(h^2+a(\gamma_{q}+u_q))^2 }\\
&= \frac{a(\gamma_{q}+u_q)}{h^2+a(\gamma_q+u_q)}\\
\end{align*}
Now when $h=1/2\sqrt{bu_q}$  we obtain that :
 
\begin{align*}
\mathbb{P}\left( \inf _{q \le i \le n} (W_i-W_q)\leq - \frac{\sqrt{bu_q}}{2} \ | \mathcal{F}_q \right) \le \frac{4a(\gamma_{q}+u_q)}{bu_q+4a(\gamma_{q}+u_q)}.
\end{align*}
As soon as $\gamma_{q}\lesssim u_q$, meaning that $-\beta \le  1-2(\beta+s)$, there exists a constant $c>0$ such that for $q$ large enough:
\begin{align*}
\mathbb{P}\left( \inf _{q \le i \le n} (W_i-W_q)\leq - \frac{\sqrt{bu_q}}{2} \ | \mathcal{F}_q \right) \le \frac{c}{b+c}.
\end{align*}
Inserting this bound in $\eqref{ineg: Sinf}$ ends the proof. 

\end{proof}

At this point, Propositions \ref{prop:T_fini} and \ref{prop:S_infini} enable us to prove that the sequence $(S_n)$ does not converge to $0$ a.s. using the same arguments as \cite{benaim1999dynamics} and then we can conclude the proof of Theorem \ref{thm:conv min}.


\begin{remark}
	When the variance of the noise sequence does not converge to $0$ ($s=0$), the previous conditions on the parameters can be summarized as : $\beta \in (1/2,1)$; $r\in (1-\beta,\beta)$, when $\beta<2/3$ and $r\in (\beta/2,\beta)$ for $\beta\ge 2/3$. 
\end{remark}
\begin{proof}[Proof of Theorem \ref{thm:conv min}] The proof is divived into two steps.

\noindent
\underline{Step 1: $S_n$ does not converge to $0$ a.s.}
Let $\mathcal{G}$ denote the event:
$$
\mathcal{G} := \left\{ \lim_{n \longrightarrow + \infty} S_n \neq  0 \right\}.
$$

The definition of $T_b^q$ implies that for any $q\in \mathbb{N}^*$ and $n\ge q$:
$$\PE[\1_{\mathcal{G}}| \mathcal{F}_n]\1_{T_b^q=n}=\PE[\1_{\mathcal{G}}| \mathcal{F}_n]\1_{T_b^q=n}\1_{S_n\ge \sqrt{bu_n}}.$$
In the meantime, if $\mathcal{S}=+\infty$ then $(S_n)$ does not converge to $0$, so $\{\mathcal{S}=+\infty\}\subset \mathcal{G}$. For $q$ large enough, such that Proposition \ref{prop:S_infini} holds, and for all $n\ge q$ :
$$\PE[\1_{\mathcal{G}}| \mathcal{F}_n]\1_{T_b^q=n}\1_{S_n\ge \sqrt{bu_n}}\ge \mathbb{P}(\mathcal{S}=+\infty|\mathcal{F}_n)\1_{T_b^q=n}\1_{S_n\ge \sqrt{bu_n}} \ge \frac{b}{c+b}\1_{T_b^q=n}\1_{S_n\ge \sqrt{bu_n}}.$$
Thus, if we consider all the integers $n$ larger than $q$, we obtain:
\begin{align*}
\PE[1_{\mathcal{G}}\ | \F_q ] &\ge \sum_{ n \ge q}\PE[\1_{\mathcal{G}}\1_{T_b^q=n}\ | \F_q]= \sum_{ n \ge q}\PE[\PE[\1_{\mathcal{G}} |\ \Fn]\1_{T_b^q=n}\ | \F_q]\\
&\ge \sum_{ n \ge q} \frac{b}{c+b} \PE[\1_{T_b^q=n}\ | \F_q]\\
&\ge  \frac{b}{c+b} \mathbb{P}(T_b^q<+\infty| \mathcal{F}_q).
\end{align*}
We now apply Proposition \ref{prop:T_fini} and obtain that:
$$ \PE[1_{\mathcal{G}}\ | \F_q ]\ge  \frac{b}{2(c+b)}>0.$$
By definition, $\mathcal{G}\subset \F_{\infty}$ so  $\lim_{q\to +\infty} \PE[\1_{\mathcal{G}} | \mathcal{F}_q]= \1_{\mathcal{G}}$ and thus the previous inequality guarantees that $ \1_{\mathcal{G}}=1$ almost surely.

\noindent
\underline{Step 2: the algorithm escapes any neighborhood of an unstable point in a finite time a.s.}

As mentioned before, we shall prove that if the algorithm is at step $n_0$ in a neighborhood $\mathcal{N}$ of a local maximum, it escapes  $\mathcal{N}$ a.s. in a finite time, meaning that $\mathbb{P}(T_{n_0}=+\infty)=0$, where  $T_{n_0}$ is the stopping time defined by $\eqref{def:T}$  \\
 Suppose that $T_{n_0}=+\infty$. In this case by definition, $X_{n+1}=\eta(\theta_{n+1},w_{n+1})-\eta(\theta_{n},\wn)$ for all $n\ge n_0$ and thus $$S_n=\eta(\theta_{n},\wn),\ \ \forall n\geq n_0$$ 
Theorem \ref{thm: cv ps} ensures that $(\theta_{n},\wn)$ converges a.s to a point $(\theta_{\infty},0)$. This together with the regularity of the function $\eta$ implies that the sequence $S_n$ goes to $\eta(\theta_{\infty},0)$ when $n\to +\infty$. \\
Since $\mathcal{N}$ is compact, the limit point $(\theta_{\infty},0)$ belongs to $\mathcal{N}$ and according to Proposition \ref{prop:instable} c) there exists $k>0$ such that: 
$$0\le k \eta(\theta_{\infty},0) \le  \langle \nabla \eta(\theta_{\infty},0),H(\theta_{\infty},0)\rangle.$$
As seen in the proof of Theorem \ref{thm: cv ps}, the limit point $(\theta_{\infty},0)$ is almost surely an equilibrium point for the dynamical system driven by $H$, so $H(\theta_{\infty},0)=0$. 
As a result we have that $$\lim_{n \longrightarrow + \infty}S_n=\eta(\theta_{\infty},0)=0.$$
From Step 1, we have seen that 
$\mathbb{P}(\lim\limits_{n \rightarrow + \infty}S_n=0)=0$, which concludes the proof.
	\end{proof}
\newcommand{\etalchar}[1]{$^{#1}$}
\providecommand{\AC}{A.-C}\providecommand{\CA}{C.-A}\providecommand{\CH}{C.-H}\providecommand{\CJ}{C.-J}\providecommand{\JC}{J.-C}\providecommand{\JP}{J.-P}\providecommand{\JB}{J.-B}\providecommand{\JF}{J.-F}\providecommand{\JJ}{J.-J}\providecommand{\JM}{J.-M}\providecommand{\KW}{K.-W}\providecommand{\PL}{P.-L}\providecommand{\RE}{R.-E}\providecommand{\SJ}{S.-J}\providecommand{\XR}{X.-R}\providecommand{\WX}{W.-X}

\end{document}